\documentclass{tlp}

\usepackage{times}
\usepackage{helvet}
\usepackage{courier}

\usepackage{amsmath, amssymb}

\usepackage{braket}

\usepackage{stmaryrd}

\usepackage[neverdecrease]{paralist}

\usepackage[mathscr]{eucal}

\RequirePackage[pdftex,usenames]{xcolor}
\definecolor{mydarkblue}{rgb}{0.0,0.08,0.45}
\usepackage[pdftex,breaklinks,colorlinks,pdfdisplaydoctitle,citecolor=mydarkblue,linkcolor=mydarkblue,urlcolor=mydarkblue,pagebackref]{hyperref}
\newtheorem{theorem}{Theorem}
\newtheorem{lemma}[theorem]{Lemma}
\newtheorem{proposition}[theorem]{Proposition}
\newtheorem{corollary}[theorem]{Corollary}

\newtheorem{definition}[theorem]{Definition}
\newtheorem{example}[theorem]{Example}

\usepackage{ifthen}
\newcommand{\brifnotempty}[1]{\ifthenelse{\equal{#1}{}}{}{ \br{#1}}}
\newenvironment{lemma*}[2][]
	{\pagebreak[2] \par \noindent \textit{Lemma~\ref{#2}\brifnotempty{#1}}.\\}{}
\newenvironment{theorem*}[2][]
	{\pagebreak[2] \par \noindent \textit{Theorem~\ref{#2}\brifnotempty{#1}}.\\}{}
\newenvironment{proposition*}[2][]
	{\pagebreak[2] \par \noindent \textit{Proposition~\ref{#2}\brifnotempty{#1}}.\\}{}
\newenvironment{corollary*}[2][]
	{\pagebreak[2] \par \noindent \textit{Corollary~\ref{#2}\brifnotempty{#1}}.\\}{}

\newcommand{\qedhere}{\mathproofbox}
\newcommand{\sqbbr}[1]{\llbracket\hspace{.2ex}#1\hspace{.2ex}\rrbracket}
\newcommand{\Sqbbr}[1]{\left\llbracket\hspace{.2ex}#1\hspace{.2ex}\right\rrbracket}

\newcommand{\llangle}{\langle\kern-.5ex\langle}
\newcommand{\llanglescr}{\langle\kern-.4ex\langle}
\newcommand{\Llangle}{\left\langle\kern-.8ex\left\langle}
\newcommand{\rrangle}{\rangle\kern-.5ex\rangle}
\newcommand{\rranglescr}{\rangle\kern-.4ex\rangle}
\newcommand{\Rrangle}{\right\rangle\kern-.8ex\right\rangle}

\newcommand{\br}[1]{(#1)}
\newcommand{\Br}[1]{\left(#1\right)}

\newcommand{\tpl}[1]{\br{#1}}

\newcommand{\lng}{n}
\newcommand{\lia}{i}
\newcommand{\lib}{j}
\newcommand{\lic}{k}

\newcommand{\sta}{\mathcal{S}}
\newcommand{\stb}{\mathcal{T}}

\newcommand{\ela}{s}
\newcommand{\elb}{t}

\newcommand{\pws}[1]{2^{#1}}
\usepackage{mathtools}

\newcommand{\trsub}[3]{#1^{[#2 \coloneqq #3]}}

\newcommand{\ccstyle}[1]{\mathrm{#1}}
\newcommand{\ccco}{\ccstyle{\operatorname{co-}}}
\newcommand{\ccP}{\ccstyle{P}}
\newcommand{\ccNP}{\ccstyle{NP}}
\newcommand{\cccoNP}{\ccco\ccNP}
\newcommand{\mymodels}{\mathrel\mid\joinrel=}
\newcommand{\ent}{\mymodels}
\newcommand{\nent}{\not\ent}
\newcommand{\eq}{\equiv}

\newcommand{\prefix}[1]{\textsf{#1}}
\newcommand{\prefm}[1]{\ensuremath{\scriptscriptstyle \mathsf{#1}}}
\newcommand{\prefixm}[1]{\prefm{#1}}

\newcommand{\pstl}[2][]{\mbox{\textsf{#1#2}}}
\newcommand{\lthen}{\supset}
\newcommand{\lequiv}{\equiv}
\newcommand{\bigland}{\bigwedge}
\newcommand{\biglor}{\bigvee}

\newcommand{\frm}{\phi}
\newcommand{\frma}{\phi}
\newcommand{\frmb}{\psi}
\newcommand{\frmu}{\mu}
\newcommand{\frmv}{\nu}
\newcommand{\twi}{I}
\newcommand{\twia}{I}
\newcommand{\twib}{J}
\newcommand{\twic}{K}
\newcommand{\twid}{L}

\newcommand{\twis}{\mathscr{I}}

\newcommand{\stwi}{\mathcal{M}}

\renewcommand{\mod}[1]{\sqbbr{#1}}

\newcommand{\atm}{p}
\newcommand{\atma}{p}
\newcommand{\atmb}{q}
\newcommand{\atmc}{r}
\newcommand{\atmd}{s}
\newcommand{\atms}{\mathscr{A}}
\newcommand{\atmsof}[1]{\mathsf{at}(#1)}

\newcommand{\slit}{S}
\newcommand{\slitp}[1][\slit]{#1^+}
\newcommand{\slitn}[1][\slit]{#1^-}
\newcommand{\lpnot}{\mathop{\sim\!}}
\newcommand{\lpif}{\leftarrow}

\newcommand{\rl}{\pi}

\newcommand{\hrl}[1][\rl]{H(#1)}

\newcommand{\hrlp}{\slitp[\hrl]}
\newcommand{\hrln}{\slitn[\hrl]}

\newcommand{\brl}[1][\rl]{B(#1)}

\newcommand{\brlp}{\slitp[\brl]}
\newcommand{\brln}{\slitn[\brl]}

\newcommand{\prg}{\mathcal{P}\hspace{-.1ex}}
\newcommand{\prga}{\mathcal{P}\hspace{-.1ex}}
\newcommand{\prgb}{\mathcal{Q}}
\newcommand{\prgc}{\mathcal{R}}
\newcommand{\prgu}{\mathcal{U}}
\newcommand{\prgv}{\mathcal{V}}
\newcommand{\tofrm}[1]{\kappa(#1)}
\newcommand{\labbu}[1]{\pstl[\small]{(B#1)}}
\newcommand{\bu}[1]{\ref{pstl:bu:#1}}
\newcommand{\smallbu}[1]{\hyperref[pstl:bu:#1]{\pstl[\scriptsize]{(B#1)}}}

\newcommand{\incorp}{\mathsf{incorporate}}

\newcommand{\oas}{\omega}
\newcommand{\oasw}{\prefix{\small W}}

\newcommand{\oasa}{\oasw}
\newcommand{\sympo}{\leq}
\newcommand{\po}[2]{\sympo^{#1}_{#2}}

\renewcommand{\pod}[1]{\po{#1}{\oas}}
\newcommand{\pow}[1]{\po{#1}{\prefixm{W}}}

\newcommand{\poa}[1]{\po{#1}{\prefixm{W}}}

\newcommand{\pop}[1]{\po{#1}{\oas'}}
\newcommand{\poo}[1]{\po{#1}{\uopr}}

\newcommand{\poeld}{\pod{\ela}}

\newcommand{\potwd}{\pod{\twia}}
\newcommand{\potww}{\pow{\twia}}

\newcommand{\symspo}{<}
\newcommand{\spo}[2]{\symspo^{#1}_{#2}}

\newcommand{\spod}[1]{\spo{#1}{\oas}}
\newcommand{\spow}[1]{\spo{#1}{\prefixm{W}}}
\newcommand{\spoa}[1]{\spo{#1}{\prefixm{W}}}

\newcommand{\spop}[1]{\spo{#1}{\oas'}}
\newcommand{\spoo}[1]{\spo{#1}{\uopr}}

\newcommand{\spotwd}{\spod{\twia}}

\newcommand{\uopb}{\mathbin{\diamond}}
\newcommand{\kb}{\mathcal{B}}

\renewcommand{\S}[2][S]{\prefix{#1}\protect\nobreakdash#2\hspace{0pt}}

\newcommand{\prgand}{\mathbin{\dot{\wedge}}}
\newcommand{\prgor}{\mathbin{\dot{\lor}}}
\newcommand{\bigprgor}{\mathop{\dot{\biglor}}}

\newcommand{\uopr}{\mathbin{\oplus}}
\newcommand{\uoprev}{\star}
\newcommand{\cn}{\mathit{Cn}}

\newcommand{\tri}{X}
\newcommand{\tria}{X}
\newcommand{\trib}{Y}
\newcommand{\tric}{Z}
\newcommand{\tris}{\mathscr{X}}
\newcommand{\twiab}{\tpl{\twia, \twib}}
\newcommand{\twibb}{\tpl{\twib, \twib}}

\newcommand{\stri}{\mathcal{M}}
\newcommand{\stria}{\mathcal{M}}
\newcommand{\strib}{\mathcal{N}}

\newcommand{\tr}{\mathsf{T}}
\newcommand{\fa}{\mathsf{F}}
\newcommand{\un}{\mathsf{U}}
\newcommand{\val}{\mathsf{V}}

\newcommand{\potrd}{\pod{\tria}}
\newcommand{\potrsd}{\pod{\tria^*}}
\newcommand{\potra}{\poa{\tria}}
\newcommand{\potrsa}{\poa{\tria^*}}

\newcommand{\potrp}{\pop{\tria}}
\newcommand{\potrsp}{\pop{\tria^*}}
\newcommand{\potro}{\poo{\tria}}
\newcommand{\potrso}{\poo{\tria^*}}

\newcommand{\spotrd}{\spod{\tria}}
\newcommand{\spotrsd}{\spod{\tria^*}}
\newcommand{\spotra}{\spoa{\tria}}
\newcommand{\spotrsa}{\spoa{\tria^*}}

\newcommand{\spotrp}{\spop{\tria}}
\newcommand{\spotrsp}{\spop{\tria^*}}
\newcommand{\spotro}{\spoo{\tria}}

\newcommand{\littleSE}[1]
	{\prefix{\scriptsize SE}\protect\nobreakdash#1\hspace{0pt}}
\newcommand{\SE}[1]{\prefix{\small SE}\protect\nobreakdash#1\hspace{0pt}}
\newcommand{\smallSE}[1]{\prefix{\footnotesize SE}\protect\nobreakdash#1\hspace{0pt}}
\newcommand{\sSE}{\prefix{\Large SE}}
\newcommand{\mSE}{\prefixm{SE}}
\newcommand{\eqSE}{\equiv_{\mSE}}
\newcommand{\entSE}{\ent_{\mSE}}
\newcommand{\nentSE}{\nent_{\mSE}}

\newcommand{\modse}[1]{\sqbbr{#1}_{\raisebox{-1pt}{\mSE}}}

\newcommand{\Modse}[1]{\Sqbbr{#1}_{\raisebox{-1pt}{\mSE}}}

\newcommand{\labpuse}[1]{\pstl[\small]{(#1)}$_{\prefix{\tiny SE}}$}
\newcommand{\puse}[1]{\ref{pstl:puse:#1}}
\newcommand{\smallpuse}[1]{\hyperref[pstl:puse:#1]{\pstl[\scriptsize]{(P#1)$_{\prefix{\tiny SE}}$}}}

\newcommand{\synt}[1]{\| #1 \|}
\submitted{4 November 2011}
\revised{6 December 2012}
\accepted{6 June 2013}

\begin{document}

\newcommand{\thetitle}
	{The Rise and Fall \\ of Semantic Rule Updates Based on \sSE-Models}
\newcommand{\thetitleshort}
	{The Rise and Fall of Semantic Rule Updates Based on \SE-Models}
\newcommand{\thetitlepdf}
	{The Rise and Fall of Semantic Rule Updates Based on SE-Models}
\newcommand{\thekeywords}
	{belief update, answer-set programs, rule update, \smallSE-models, support, literal inertia}
\newcommand{\thekeywordspdf}
	{belief update, answer-set programs, rule update, SE-models, support, literal inertia}

\pdfinfo{
	/Title (\thetitlepdf{})
	/Author (Martin Slota and Joao Leite)
	/Keywords (\thekeywordspdf{})
}

\title
	[\thetitleshort{}]
	{\thetitle{}$^{\footnotemark[1]}$}
\author
	[M.\ Slota and J.\ Leite]
	{
%		MARTIN SLOTA$^{\footnotemark[1]}$ and JO{\~A}O LEITE \\
		MARTIN SLOTA and JO{\~A}O LEITE \\
		CENTRIA \& Departamento de Inform{\'a}tica \\
		Universidade Nova de Lisboa \\
		2829-516 Caparica, Portugal
	}

\renewcommand{\thefootnote}{\fnsymbol{footnote}}
\footnotetext[1]{This is an extended version of \cite{Slota2010b}.}
\setcounter{footnote}{0}
\renewcommand{\thefootnote}{\arabic{footnote}}

\maketitle

\begin{abstract}
	Logic programs under the stable model semantics, or answer-set programs,
	provide an expressive rule-based knowledge representation framework,
	featuring a formal, declarative and well-understood semantics. However,
	handling the evolution of rule bases is still a largely open problem. The
	AGM framework for belief change was shown to give inappropriate results when
	directly applied to logic programs under a non-monotonic semantics such as
	the stable models. The approaches to address this issue, developed so far,
	proposed update semantics based on manipulating the syntactic structure of
	programs and rules.

	More recently, AGM revision has been successfully applied to a significantly
	more expressive semantic characterisation of logic programs based on
	\smallSE-models. This is an important step, as it changes the focus from the
	evolution of a syntactic representation of a rule base to the evolution of
	its semantic content.

	In this paper, we borrow results from the area of belief update to tackle
	the problem of updating (instead of revising) answer-set programs. We prove
	a representation theorem which makes it possible to constructively define
	any operator satisfying a set of postulates derived from Katsuno and
	Mendelzon's postulates for belief update. We define a specific operator
	based on this theorem, examine its computational complexity and compare the
	behaviour of this operator with syntactic rule update semantics from the
	literature. Perhaps surprisingly, we uncover a serious drawback of all rule
	update operators based on Katsuno and Mendelzon's approach to update and on
	\smallSE-models.
\end{abstract}


\section{Introduction}

Answer-Set Programming (ASP) \cite{Gelfond1988,Baral2003} is now widely
recognised as a valuable approach to knowledge representation and reasoning,
mostly due to its simple and well-understood declarative semantics, its rich
expressive power, and the existence of efficient implementations.

However, the dynamic character of many applications that can benefit from ASP
calls for the development of ways to deal with the evolution of answer-set
programs and the inconsistencies that may arise.

The problems associated with knowledge evolution have been extensively
studied, over the years, by different research communities, namely in the
context of Classical Logic, and in the context of Logic Programming.

The former have been inspired, to a large extent, by the seminal work of
Alchourr{\'{o}}n, G{\"{a}}rdenfors and Makinson (AGM) who proposed a set of
desirable properties of belief change operators, now called \emph{AGM
postulates} \cite{Alchourron1985}. Subsequently, \emph{update} and
\emph{revision} have been distinguished as two very related but ultimately
different belief change operations \cite{Keller1985,Winslett1990,Katsuno1991}.
While revision deals with incorporating new information about a \emph{static}
world, update takes place when changes occurring in a \emph{dynamic} world are
recorded. Katsuno and Mendelzon formulated a separate set of postulates for
update, now known as \emph{KM postulates}.

Both AGM and KM postulates were later studied in the context of Logic
Programming, only to find that their formulations based on a non-monotonic
semantics, such as the answer sets, are inappropriate \cite{Eiter2002}. Like
many belief change operators, earlier methods used to tackle rule updates were
based on literal inertia \cite{Alferes1996} but proved not sufficiently
expressive. This led to the development of rule update semantics based on
different intuitions, principles and constructions, when compared to their
classical counterparts. For example, the introduction of the \emph{causal
rejection principle} \cite{Leite1997} motivated a line of work on several rule
update semantics
\cite{Alferes2000,Eiter2002,Leite2003,Alferes2005,Osorio2007}, all of them
with a strong syntactic flavour. Other approaches tackle rule updates by
employing syntactic transformations and other methods, such as abduction
\cite{Sakama2003}, forgetting \cite{Zhang2005}, prioritisation
\cite{Zhang2006}, preferences \cite{Delgrande2007}, or dependencies on default
assumptions \cite{Sefranek2006,Krumpelmann2010,Sefranek2011}.

Though useful in a number of practical scenarios
\cite{Alferes2003,Saias2004,Siska2006,Ilic2008,Slota2011a}, it turned out that
most of these semantics exhibit undesirable behaviour. For example, except for
the semantics proposed in \cite{Alferes2005,Sefranek2011}, a tautological
update may influence the result under all of these semantics, a behaviour that
is highly undesirable when considering knowledge updates. Other kinds of
irrelevant updates are even more problematic and subject of ongoing research
\cite{Sefranek2006,Sefranek2011}. But more important, the common feature of
all of these semantics is that they make heavy use of the syntactic structure
of programs and rules, making any analysis of their semantic properties a
daunting task.

Recently, AGM revision was reformulated in the context of Logic Programming in
a manner analogous to belief revision in classical propositional logic, and
specific revision operators for logic programs were investigated
\cite{Delgrande2008,Osorio2007}. Central to this novel approach are
\emph{\SE-models} \cite{Turner2003} which provide a monotonic semantic
characterisation of logic programs that is strictly more expressive than the
answer-set semantics. Furthermore, two programs have the same set of
\SE-models if and only if they are strongly equivalent \cite{Lifschitz2001},
which means that programs $\prga, \prgb$ with the same set of \SE-models can
be modularly replaced by one another, even in the presence of additional
rules, without affecting the resulting answer sets.

Indeed, these results constitute an important breakthrough in the research of
answer-set program evolution. They change the focus from the syntactic
representation of a program, where not all rules and literal occurrences are
necessarily relevant to the meaning of the program as a whole, to its semantic
content, i.e.\ to the information that the program is intended to represent.

In this paper, we follow a similar path, but to tackle the problem of
answer-set program \emph{updates}, instead of \emph{revision} as in
\cite{Delgrande2008}.

Using \emph{\SE-models}, we adapt the KM postulates to answer-set program
updates and prove a representation theorem that provides a constructive
characterisation of rule update operators satisfying the postulates, making it
possible to define and evaluate any operator satisfying the postulates using
an intuitive construction. We show how this constructive characterisation can
be used by defining a concrete answer-set program update operator that can be
seen as a counterpart of Winslett's belief update operator \cite{Winslett1990}
which satisfies the KM postulates and is commonly used in the literature.

However, while investigating the operator's properties, we uncover a serious
drawback which, as it turns out, extends to all answer-set program update
operators based on \SE-models and Katsuno and Mendelzon's approach to updates.
This finding is very important as it guides the research on updates of
answer-set programs away from the purely semantic approach materialised in AGM
and KM postulates or, alternatively, to the development of semantic
characterisations of answer-set programs, richer than \SE-models, that are
appropriate for describing their dynamic behaviour.

The remainder of this paper is structured as follows: In
Section~\ref{sect:preliminaries} we introduce the formal concepts that are
necessary throughout the rest of the paper. Section~\ref{sect:km operators}
contains the reformulation of KM postulates for logic program updates and the
representation theorem that establishes a general constructive
characterisation of rule update operators obeying the postulates. We also show
how this theorem can be used by defining a specific rule update operator that
satisfies the postulates and we examine the computational complexity of query
answering for this operator. In Section~\ref{sect:comparison} we further
analyse the previously defined operator and establish that all semantic rule
update operators based on \SE-models exhibit an undesired behaviour.
% We proceed by discussing additional principles for rule change based on
% \SE-models in Section~\ref{sect:properties based on se-models}.
We summarise our findings in Section~\ref{sect:conclusion}.

\section{Preliminaries}

\label{sect:preliminaries}

We consider a propositional language over a finite set of propositional
variables $\atms$ and the usual set of propositional connectives to form
propositional formulae. An \emph{objective literal} is either an atom $\atm$
or its negation $\lnot \atm$. A \emph{Horn clause} is a disjunction of at most
one atom and zero or more negated atoms; a \emph{Horn formula} is a
conjunction of Horn clauses.

A (propositional) interpretation is any subset of $\atms$ and the set of all
interpretations is $\twis = \pws{\atms}$. We use the standard semantics for
propositional formulae and denote the set of models of a formula $\frm$ by
$\mod{\frm}$. We also write $\twib \ent \frm$ if $\twib \in \mod{\frm}$. We
say that a formula $\frm$ is \emph{complete} if $\mod{\frm}$ is a singleton
set. For formulae $\frma$, $\frmb$ we say that \emph{$\frma$ is equivalent to
$\frmb$}, denoted by $\frma \eq \frmb$, if $\mod{\frma} = \mod{\frmb}$, and
that \emph{$\frma$ entails $\frmb$}, denoted by $\frma \ent \frmb$, if
$\mod{\frma} \subseteq \mod{\frmb}$. As we are dealing with the finite case,
every knowledge base can be expressed by a single formula.

\subsection{Belief Update}

Update is a belief change operation that brings a knowledge base \emph{up to
date} when the \emph{world described by it changes}
\cite{Keller1985,Katsuno1991}. Formally, a belief update operator is a
function that takes two formulae, representing the original knowledge base and
its update, as arguments and returns a formula representing the updated
knowledge base. To further specify the desired properties of update operators,
the following eight postulates for a belief update operator $\uopb$ and
formulae $\frma$, $\frmb$, $\frmu$, $\frmv$ were proposed in
\cite{Katsuno1991}:

\begin{enumerate}
	\renewcommand{\theenumi}{\labbu{\arabic{enumi}}}
	\renewcommand{\labelenumi}{\theenumi\hfill}
	\setlength{\itemsep}{.5ex}

	\item $\frma \uopb \frmu \ent \frmu$.
		\label{pstl:bu:1}

	\item If $\frma \ent \frmu$, then $\frma \uopb \frmu \eq \frma$.
		\label{pstl:bu:2}

	\item If $\mod{\frma} \neq \emptyset$ and $\mod{\frmu} \neq \emptyset$, then
		$\mod{\frma \uopb \frmu} \neq \emptyset$.
		\label{pstl:bu:3}

	\item If $\frma \eq \frmb$ and $\frmu \eq \frmv$, then $\frma \uopb
		\frmu \eq \frmb \uopb \frmv$.
		\label{pstl:bu:4}

	\item $(\frma \uopb \frmu) \land \frmv \ent \frma \uopb (\frmu \land
		\frmv)$.
		\label{pstl:bu:5}

	\item If $\frma \uopb \frmu \ent \frmv$ and $\frma \uopb \frmv \ent \frmu$,
		then $\frma \uopb \frmu \eq \frma \uopb \frmv$.
		\label{pstl:bu:6}

	\item If $\frma$ is complete, then $(\frma \uopb \frmu) \land (\frma \uopb
		\frmv) \ent \frma \uopb (\frmu \lor \frmv)$.
		\label{pstl:bu:7}

	\item $(\frma \lor \frmb) \uopb \frmu \eq (\frma \uopb \frmu) \lor (\frmb
		\uopb \frmu)$.
		\label{pstl:bu:8}
\end{enumerate}

Katsuno and Mendelzon also proved an important representation theorem that
makes it possible to define and evaluate any operator satisfying these
postulates using an intuitive construction. It is based on treating the models
of a knowledge base as possible real states of the modelled world. An update
of an original knowledge base $\frma$ is performed by modifying each of its
models as little as possible to make it consistent with new information in the
update $\frmu$, obtaining a new set of interpretations -- the models of the
updated knowledge base. More formally,
\[
	\mod{\frma \uopb \frmu}
	=
	\bigcup_{\twi \in \mod{\frma}}
	\incorp(\mod{\frmu}, \twi)
	\enspace,
\]
where $\incorp(\stwi, \twi)$ returns the members of $\stwi$ closer to $\twi$.
A natural way of defining $\incorp(\stwi, \twi)$ is by assigning an order
$\po{\twi}{}$ over $\twis$ to each interpretation $\twi$ and taking the minima
of $\stwi$ w.r.t.\ $\po{\twi}{}$, i.e.\ $\incorp(\stwi, \twi) = \min(\stwi,
\po{\twi}{})$. In the following we first formally establish the concept of an
\emph{order assignment}; thereafter we define when an update operator is
\emph{characterised by} such an assignment.

Given a set $\sta$, a \emph{preorder over $\sta$} is a reflexive and
transitive binary relation over $\sta$; a \emph{strict preorder over $\sta$}
is an irreflexive and transitive binary relation over $\sta$; a \emph{partial
order over $\sta$} is a preorder over $\sta$ that is antisymmetric. Given a
preorder $\po{}{}$ over $\sta$, we denote by $\spo{}{}$ the strict preorder
induced by $\po{}{}$, i.e.\ $\ela \spo{}{} \elb$ if and only if $\ela \po{}{}
\elb$ and not $\elb \po{}{} \ela$. For any subset $\stb$ of $\sta$, the set of
\emph{minimal elements of $\stb$ w.r.t.\ $\po{}{}$} is
\[
	\min(\stb, \po{}{})
	=
	\Set{
		\ela \in \stb
		|
		\lnot \exists \elb \in \stb : \elb \spo{}{} \ela
	}
	\enspace.
\]

\begin{definition}
	[Order assignment]
	Let $\sta$ be a set. A \emph{preorder assignment over $\sta$} is any
	function $\oas$ that assigns a preorder $\poeld$ over $\sta$ to each $\ela
	\in \sta$. A \emph{partial order assignment over $\sta$} is any preorder
	assignment $\oas$ over $\sta$ such that $\poeld$ is a partial order over
	$\sta$ for every $\ela \in \sta$.
\end{definition}

\begin{definition}
	[Belief update operator characterised by an order assignment]
	Let $\uopb$ be a belief update operator and $\oas$ a preorder assignment
	over $\twis$. We say that $\uopb$ is \emph{characterised by $\oas$} if for
	all formulae $\frma$, $\frmu$,
	\begin{equation}
		\label{eq:bu operator}
		\mod{\frma \uopb \frmu}
		=
		\bigcup_{\twi \in \mod{\frma}}
		\min \Br{ \mod{\frmu}, \potwd }
		\enspace.
	\end{equation}
\end{definition}

A natural condition to impose on the assigned orders is that every
interpretation be the closest to itself. This is captured by the notion of a
\emph{faithful} order assignment:

\begin{definition}
	[Faithful order assignment \cite{Katsuno1991}]
	A preorder assignment $\oas$ over $\twis$ is \emph{faithful} if for every
	interpretation $\twia$ the following condition is satisfied:
	\[
		\text{For every } \twib \in \twis \text{ with } \twib \neq \twia
		\text{ it holds that } \twia \spotwd \twib
		\enspace.
	\]
\end{definition}

The representation theorem of \cite{Katsuno1991} states that operators
characterised by faithful order assignments are exactly those that satisfy the
KM postulates.

\begin{theorem}
	[Representation theorem for belief updates \cite{Katsuno1991}]
	\label{thm:bu:representation}
	Let $\uopb$ be a belief update operator. Then the following conditions are
	equivalent:
	\begin{enumerate}[a)]
		\item The operator $\uopb$ satisfies conditions \bu{1} -- \bu{8}.

		\item The operator $\uopb$ is characterised by a faithful preorder
			assignment.

		\item The operator $\uopb$ is characterised by a faithful partial order
			assignment.
	\end{enumerate}
\end{theorem}

Katsuno and Mendelzon's results provide a framework for belief update
operators, each specified on the semantic level by a faithful partial order
assignment over $\twis$. The most influential instance of this framework is
the \emph{Possible Models Approach} \cite{Keller1985,Winslett1990}, also
referred to as \emph{Winslett's belief update semantics}, based on minimising
the set of atoms whose truth value changes when an interpretation is updated.
Formally, Winslett's partial order assignment $\oasw$ is defined for all
interpretations $\twia$, $\twib$, $\twic$ by
\begin{align*}
	& \twib \potww \twic
	&& \text{ if and only if }
	&& (\twib \div \twia) \subseteq (\twic \div \twia)
	\enspace,
\end{align*}
where $\div$ denotes set-theoretic symmetric difference. It is not difficult
to verify that $\oasw$ is a faithful partial order assignment, so it follows
from Theorem~\ref{thm:bu:representation} that any belief update operator
$\uopb$ characterised by $\oasw$ satisfies postulates \bu{1} -- \bu{8}. Note
that there is a whole class of operators characterised by $\oasw$ that differ
in the syntactic representation of updated belief bases. Insofar as we are
interested in the semantic properties of Winslett's updates, it follows from
\bu{4} that it does not matter which operator from this class we pick. This is
illustrated in the following example:

\begin{example}
	[Winslett's belief update semantics]
	Consider the knowledge base $\frma = (\atma \land (\atmb \lequiv \atmc))$
	and the update $\frmu = (\atmb \lor \atmc)$ over the set of atoms $\atms =
	\set{\atma, \atmb, \atmc}$. Their sets of models are as follows:
	\begin{align*}
		\mod{\frma} &= \set{
			\set{\atma},
			\set{\atma, \atmb, \atmc}
		} \enspace, \\
		\mod{\frmu} &= \set{
			\set{\atmb},
			\set{\atmc},
			\set{\atmb, \atmc},
			\set{\atma, \atmb},
			\set{\atma, \atmc},
			\set{\atma, \atmb, \atmc}
		}
		\enspace.
	\end{align*}
	When performing an update of $\frma$ by $\frmu$ under Winslett's update
	semantics, equation \eqref{eq:bu operator} applies as follows:
	\[
		\mod{\frma \uopb \frmu}
		=
		\bigcup_{\twia \in \mod{\frma}} \min \Br{ \mod{\frmu}, \potww }
		=
		\min \Br{ \mod{\frmu}, \pow{\set{\atma}} }
		\cup
		\min \Br{ \mod{\frmu}, \pow{\set{\atma, \atmb, \atmc}} }
		\enspace.
	\]
	The models of $\frmu$ that ``differ least'' from $\set{\atma}$, in the sense
	of the order assignment $\oasw$, are $\set{\atma, \atmb}$ and $\set{\atma,
	\atmc}$. Furthermore, since $\oasw$ is faithful, the unique model of $\frmu$
	that is minimally distant from $\set{\atma, \atmb, \atmc}$ is $\set{\atma,
	\atmb, \atmc}$ itself. Consequently,
	\[
		\mod{\frma \uopb \frmu}
		=
		\set{
			\set{\atma, \atmb},
			\set{\atma, \atmc},
			\set{\atma, \atmb, \atmc}
		}
		\enspace.
	\]
	Note that from the syntactic viewpoint, $\frma \uopb \frmu$ can be any
	formula with the above set of models. Thus, it may for example be the case
	that $\frma \uopb \frmu = (\atma \land (\atmb \lor \atmc))$ while for
	another operator $\uopb'$, also characterised by $\oasw$, $\frma \uopb'
	\frmu = ((\atma \land \atmb) \lor (\atma \land \atmc))$.
\end{example}

\subsection{Computational Complexity of Winslett's Update Semantics}

Computationally, query answering for Winlett's operator, i.e.\ the problem of
deciding whether $\frma \uopb \frmu \ent \frmb$, where $\uopb$ is
characterised by $\oasw$, belongs to the second level of the polynomial
hierarchy \cite{Eiter1992}. We formulate this result formally as it later
facilitates the study of computational complexity of a newly introduced rule
update operator.

Assuming that the reader is familiar with the classes $\ccNP$ and $\cccoNP$,
we briefly introduce the \emph{polynomial hierarchy}
\cite{Meyer1972,Stockmeyer1976}. Its definition relies on the notion of an
\emph{oracle}: An oracle for a class of decision problems $C$ can decide any
problem in $C$ in just one step of computation. We denote by $\ccNP^C$ the
class of decision problems solvable in polynomial time by a non-deterministic
Turing machine that can make calls to an oracle for $C$. The classes
$\Sigma^\ccP_\lia$ and $\Pi^\ccP_\lia$ of the polynomial hierarchy are defined
inductively as follows: $\Sigma^\ccP_0 = \Pi^\ccP_0 = \ccP$ and for all $\lia
\geq 0$,
\begin{align*}
	\Sigma^\ccP_{\lia + 1} &= \ccNP^{\Sigma^\ccP_\lia}
	&& \text{and}
	& \Pi^\ccP_{\lia + 1} &= \ccco\Sigma^\ccP_{\lia + 1}
	\enspace.
\end{align*}
In the general case, query answering for Winslett's updates is
$\Pi^\ccP_2$-complete.

\begin{theorem}
	[Part of Theorem~6.4 in \cite{Eiter1992}]
	\label{thm:oasw complexity general}
	Let $\uopb$ be a belief update operator characterised by $\oasw$. Deciding
	whether $\frma \uopb \frmu \ent \frmb$ for formulae $\frma$, $\frmu$,
	$\frmb$ is $\Pi^\ccP_2$-complete. Hardness holds even if $\frma$ is a
	conjunction of atoms and $\frmb$ is one of the atoms in that conjunction.
\end{theorem}
However, when dealing only with Horn formulae, the problem drops to the first
level of the polynomial hierarchy:

\begin{theorem}
	[Part of Theorem~7.2 in \cite{Eiter1992}]
	\label{thm:oasw complexity Horn}
	Let $\uopb$ be a belief update operator characterised by $\oasw$. Deciding
	whether $\frma \uopb \frmu \ent \frmb$ for Horn formulae $\frma$, $\frmu$,
	$\frmb$ is $\cccoNP$-complete. Hardness holds even if $\frma$ is a
	conjunction of objective literals and $\frmb$ is one of the literals in that
	conjunction.
\end{theorem}

\subsection{Logic Programming}

We define the syntax and semantics of logic programs, borrowing some of the
notation used in \cite{Delgrande2008}.

An \emph{atom} is any $\atm \in \atms$. A \emph{literal} is an atom $\atm$ or
its default negation $\lpnot \atm$. Given a set of literals $\slit$, we
introduce the following notation:
\begin{align*}
	\slitp &= \Set{\atm \in \atms | \atm \in \slit} \enspace,
	&
	\slitn &= \Set{\atm \in \atms | \lpnot \atm \in \slit} \enspace,
	\\
	\lpnot \slit &= \Set{\lpnot \atm | \atm \in \slit \cap \atms} \enspace,
	&
	\lnot \slit &= \Set{\lnot \atm | \atm \in \slit \cap \atms}
	\enspace.
\end{align*}
A \emph{rule} is a pair of sets of literals $\rl = \tpl{\hrl, \brl}$. We say
that $\hrl$ is the \emph{head of $\rl$} and $\brl$ is the \emph{body of
$\rl$}. Usually, for convenience, we write $\rl$ as
\begin{equation}
	\label{eq:rule}
	\hrlp; \lpnot \hrln \lpif \brlp, \lpnot \brln.
\end{equation}
Operators `;' and `,' express disjunctive and conjunctive connectives,
respectively. A rule is called a \emph{fact} if its head contains exactly one
literal and its body is empty. A fact is \emph{positive} if the literal in its
head is an atom. A rule is \emph{non-disjunctive} if its head contains at most
one literal; \emph{definite} if it is non-disjunctive and its head and body
contain only atoms. A \emph{program} is a set of rules. A program is
\emph{non-disjunctive} if all rules inside it are non-disjunctive;
\emph{definite} if all rules inside it are definite.

Turning to the semantics, we need to define \emph{answer sets} and
\emph{\SE-models} of a logic program. We start by defining the more basic
notion of a \emph{(classical) model} of a logic program. For every rule $\rl$
of the form \eqref{eq:rule} we denote by $\tofrm{\rl}$ the propositional
formula
\[
	\bigland (\brlp \cup \lnot \brln)
	\lthen
	\biglor (\hrlp \cup \lnot \hrln)
	\enspace.
\]
For a program $\prg$, $\tofrm{\prg} = \bigland_{\rl \in \prg} \tofrm{\rl}$.
An interpretation $\twib$ is a \emph{model} of a program $\prg$, denoted by
$\twib \ent \prg$, if $\twib \ent \tofrm{\prg}$.
% The set of all models of a program $\prg$ is denoted by $\mod{\prg}$.
We say that $\prg$ is \emph{consistent} if it has some classical model.

An interpretation $\twib$ is an \emph{answer set} of a program $\prg$ if it is
a subset-minimal model of the \emph{reduct of $\prg$ relative to $\twib$}:
\begin{align*}
	\prg^\twib = \Set{
		\hrl^+ \lpif \brl^+.
		|
		\rl \in \prg
		\land
		\hrl^- \subseteq \twib
		\land
		\brl^- \cap \twib = \emptyset
	} \enspace.
\end{align*}
% We denote the set of answer sets of $\prg$ by $\modsm{\prg}$.

\emph{\SE-models} \cite{Turner2003}, based on the non-classical logic of
Here-and-There \cite{Heyting1930,Lukasiewicz1941,Pearce1997}, provide a
monotonic characterisation of logic programs that is expressive enough to
capture both their classical models and answer sets. We use \SE-models in the
following sections to reformulate the KM postulates for belief update in the
context of rule updates.

Intuitively, each \SE-interpretation assigns \emph{one of three truth values}
to every atom. Technically it consists of a pair of propositional
interpretations, the first containing atoms that are true and the second
containing atoms that are not false. Formally:

\begin{definition}
	[\SE-interpretation \cite{Turner2003}]
	An \emph{\SE-interpretation} is a pair of interpretations $\twiab$ such that
	$\twia \subseteq \twib$. The set of all \SE-interpretations is denoted by
	$\tris$. 	
\end{definition}
\SE-models themselves are defined by referring to the \emph{program reduct}
used to define answer sets above.

\begin{definition}
	[\SE-model \cite{Turner2003}]
	\label{def:se-models}
	Let $\prg$ be a program. An \SE-interpretation $\twiab$ is an
	\emph{\SE-model of $\prg$} if $\twib \ent \prg$ and $\twia \ent \prg^\twib$.
	The set of all \SE-models of $\prg$ is denoted by $\modse{\prg}$ and we
	write $\twiab \ent \prg$ if $\twiab \in \modse{\prg}$.
\end{definition}
Note that $\twib \ent \prg$ if and only if $\twibb \in \modse{\prg}$, so
\SE-models capture the classical models of a program. And just like classical
models, the set of \SE-models of a program is \emph{monotonic}, i.e.\ larger
programs have smaller sets of \SE-models. This is one of the important
differences between \SE-models and the non-monotonic answer sets.

Nevertheless, a program's answer sets, just like its classical models, can be
extracted from its set of \SE-models: An interpretation $\twib$ is an answer
set of $\prg$ if and only if $\twibb \in \modse{\prg}$ and no $\twiab \in
\modse{\prg}$ with $\twia \subsetneq \twib$ exists. This implies that programs
with the same set of \SE-models also have the same answer sets. Moreover, when
such programs are augmented with the same set of rules, the resulting programs
\emph{still} have the same answer sets. In many situations such a property is
desirable as it allows one program to be modularly replaced by another one,
even in the presence of additional rules, without affecting the resulting
answer sets. It is typically referred to as \emph{strong equivalence}
\cite{Lifschitz2001} and the relationship between \SE-models and strong
equivalence is formally captured as follows:

\begin{proposition}
	[\SE-models and strong equivalence \cite{Turner2003}]
	\label{prop:se:strong equivalence}
	Let $\prga$, $\prgb$ be programs. It holds that $\modse{\prga} =
	\modse{\prgb}$ if and only if for every program $\prgc$, the answer sets of
	$\prga \cup \prgc$ and $\prgb \cup \prgc$ are the same.
\end{proposition}
In other words, \SE-models exactly capture the concept of strong equivalence.
This also explains the origin of the name \emph{\SE-models} -- ``\SE{}''
stands for \emph{strong equivalence}. Based on this result, we define strong
equivalence and entailment as follows:

\begin{definition}
	[Strong equivalence and strong entailment]
	Let $\prga$, $\prgb$ be programs. We say that \emph{$\prga$ is strongly
	equivalent to $\prgb$}, denoted by $\prga \eqSE \prgb$, if $\modse{\prga} =
	\modse{\prgb}$, and that \emph{$\prga$ strongly entails $\prgb$}, denoted by
	$\prga \entSE \prgb$, if $\modse{\prga} \subseteq \modse{\prgb}$.
\end{definition}

An important distinguishing property of \SE-models that we will need to
carefully consider in the following sections is that whenever a program $\prg$
has the \SE-model $\twiab$, it also has the \SE-model $\twibb$. More
generally, any set of \SE-interpretations with this property is referred to as
\emph{well-defined} \cite{Delgrande2008}.

\begin{definition}
	[Well-defined set of \SE-interpretations \cite{Delgrande2008}]
	\label{def:se:well-defined}
	For every \SE-interpretation $\tri = \twiab$ we denote by $\tri^*$ the
	\SE-interpretation $\twibb$. A set of \SE-interpretations $\stri$ is
	\emph{well-defined} if for every \SE-interpretation $\tri$, $\tri \in \stri$
	implies $\tri^* \in \stri$.
\end{definition}

In fact, as pinpointed in the following result, not only is the set of
\SE-models of a program well-defined, but every well-defined set of
\SE-interpretations is also the set of \SE-models of some program.

\begin{proposition}
	[\citeNP{Delgrande2008}]
	\label{prop:se:well-defined}
	A set of \SE-interpretations $\stri$ is well-defined if and only if $\stri =
	\modse{\prg}$ for some program $\prg$.
\end{proposition}

As a consequence, whenever $\twia \subsetneq \twib$, there is no program that
has the single \SE-model $\tri = \twiab$, though there is a program that has
the pair of \SE-models $\tri$, $\tri^*$. The following notion of a \emph{basic
program} is thus analogous to the concept of a \emph{complete formula} that is
used in the formulation of belief update postulate \bu{7}.

\begin{definition}
	[Basic program]
	\label{def:se:basic program}
	We say that a program $\prg$ is \emph{basic} if $\modse{\prg} = \set{\tri,
	\tri^*}$ for some \SE-interpretation $\tri$.
\end{definition}
Note that a program is basic if either it has a unique \SE-model $\twibb$, or
a pair of \SE-models $\twiab$ and $\twibb$. In the former case, the program
exactly determines the truth values of all atoms -- the atoms in $\twib$ are
true and the remaining atoms are false. In the latter case, the program makes
atoms in $\twia$ true, the atoms in $\twib \setminus \twia$ may either be
undefined or true, as long as they all have the same truth value, and the
remaining atoms are false.

\section{Semantic Rule Updates Based on \SE-Models}

\label{sect:km operators}

With the necessary concepts defined, we are ready to step forward and tailor
the belief update postulates and operators to the context of logic programs
viewed through their sets of \SE-models. Since \SE-models provide a
\emph{monotonic} characterisation of logic programs, the analysis provided in
\cite{Eiter2002}, which showed KM postulates not appropriate for use with
non-monotonic semantics, no longer applies. In the following we reformulate
the belief update postulates as well as a constructive characterisation of
semantic rule update operators, and finally show a counterpart of the
representation theorem for belief updates. The studied operators are semantic
in their very nature and in line with KM postulates, in contrast with the
traditional syntax-based approaches to rule updates
\cite{Leite1997,Alferes2000,Eiter2002,Sakama2003,Zhang2005,Alferes2005,Zhang2006,Delgrande2007,Sefranek2011}.

Similarly as in the case of belief updates, we liberally define a rule update
operator as any function that takes two inputs, the original program and its
update, and returns the updated program.

\begin{definition}
	[Rule update operator]
	A \emph{rule update operator} is a binary function on the set of all
	programs.
\end{definition}

In order to reformulate postulates \bu{1} -- \bu{8} for logic programs under
the \SE-model semantics, we first need to specify what a conjunction and
disjunction of logic programs is. To this end, we introduce program
conjunction and disjunction operators. These are required to assign, to each
pair of programs, a program whose set of \SE-models is the intersection and
union, respectively, of the sets of \SE-models of argument programs.

\begin{definition}
	[Program conjunction and disjunction]
	A binary operator $\prgand$ on the set of all programs is a \emph{program
	conjunction operator} if for all programs $\prga$, $\prgb$,
	\[
		\modse{\prga \prgand \prgb} = \modse{\prga} \cap \modse{\prgb} \enspace.
	\]
	A binary operator $\prgor$ on the set of all programs is a \emph{program
	disjunction operator} if for all programs $\prga$, $\prgb$,
	\[
		\modse{\prga \prgor \prgb} = \modse{\prga} \cup \modse{\prgb} \enspace.
	\]
\end{definition}

In the following we assume that some program conjunction and disjunction
operators $\prgand$, $\prgor$ are given. Note that the program conjunction
operator may simply return the union of argument programs; it is the same as
the \emph{expansion operator} defined in \cite{Delgrande2008}. A program
disjunction operator can be defined by translating the argument programs into
the logic of Here-and-There \cite{Heyting1930,Lukasiewicz1941,Pearce1997},
taking their disjunction and transforming the resulting formula back into a
logic program (using results from \cite{Cabalar2007}).

The final obstacle before we can proceed with introducing the new postulates
is the following: We need to substitute the notion of a \emph{complete
formula} used in \bu{7} with a suitable class of logic programs. It turns out
that the notion of a \emph{basic program}, as introduced in
Definition~\ref{def:se:basic program}, is a natural candidate for this
purpose. While a complete formula is defined as having a unique model, a
program is basic if it has either a unique \SE-model $\twibb$, or a pair of
\SE-models $\twiab$ and $\twibb$. The latter case needs to be allowed in order
to make the new postulate applicable to \SE-interpretations $\twiab$ with
$\twia \subsetneq \twib$ because no program has the single \SE-model $\twiab$
(c.f.\ Proposition~\ref{prop:se:well-defined}).

The following are the reformulated postulates for a rule update operator
$\uopr$ and programs $\prga$, $\prgb$, $\prgu$, $\prgv$:
\begin{enumerate}
	\renewcommand{\theenumi}{\labpuse{P\arabic{enumi}}}
	\renewcommand{\labelenumi}{\theenumi\hfill}
	\setlength{\itemsep}{.5ex}

	\item $\prga \uopr \prgu \entSE \prgu$.
		\label{pstl:puse:1}

	\item If $\prga \entSE \prgu$, then $\prga \uopr \prgu \eqSE \prga$.
		\label{pstl:puse:2}

	\item If $\modse{\prga} \neq \emptyset$ and $\modse{\prgu} \neq \emptyset$,
		then $\modse{\prga \uopr \prgu} \neq \emptyset$.
		\label{pstl:puse:3}

	\item If $\prga \eqSE \prgb$ and $\prgu \eqSE \prgv$, then $\prga \uopr
		\prgu \eqSE \prgb \uopr \prgv$.
		\label{pstl:puse:4}

	\item $(\prga \uopr \prgu) \prgand \prgv \entSE \prga \uopr (\prgu \prgand
		\prgv)$.
		\label{pstl:puse:5}

	\item If $\prga \uopr \prgu \entSE \prgv$ and $\prga \uopr \prgv \entSE
		\prgu$, then $\prga \uopr \prgu \eqSE \prga \uopr \prgv$.
		\label{pstl:puse:6}

	\item If $\prga$ is basic, then $(\prga \uopr \prgu) \prgand (\prga \uopr
		\prgv) \entSE \prga \uopr (\prgu \prgor \prgv)$.
		\label{pstl:puse:7}

	\item $(\prga \prgor \prgb) \uopr \prgu \eqSE (\prga \uopr \prgu) \prgor
		(\prgb \uopr \prgu)$.
		\label{pstl:puse:8}
\end{enumerate}

Now we turn to a constructive characterisation of rule update operators
satisfying conditions \puse{1} -- \puse{8}. Analogically to belief updates, it
is based on an order assignment, but this time over the set of all
\SE-interpretations $\tris$. Since the set of \SE-models of a program
must be well-defined, not every order assignment characterises a rule update
operator. We thus additionally define \emph{well-defined order assignments}
as those that do.

\begin{definition}
	[Rule update operator characterised by an order assignment]
	Let $\uopr$ be a rule update operator and $\oas$ a preorder assignment
	over $\tris$. We say that $\uopr$ is \emph{characterised by $\oas$} if for
	all programs $\prga$, $\prgu$,
	\[
		\modse{\prga \uopr \prgu}
		=
		\bigcup_{\tria \in \modse{\prga}}
		\min \left( \modse{\prgu}, \potrd \right)
		\enspace.
	\]
	We say that a preorder assignment over $\tris$ is \emph{well-defined} if
	some rule update operator is characterised by it.
\end{definition}

Similarly as with belief update, we require the order assignment to be
faithful, i.e.\ to consider each \SE-interpretation the closest to itself.

\begin{definition}
	[Faithful order assignment]
	A preorder assignment $\oas$ over $\tris$ is \emph{faithful} if for every
	\SE-interpretation $\tria$ the following condition is satisfied:
	\[
		\text{For every } \trib \in \tris \text{ with } \trib \neq \tria
		\text{ it holds that } \tria \spotrd \trib
		\enspace.
	\]
\end{definition}

Interestingly, faithful assignments characterise the same class of operators
as the larger class of semi-faithful assignments, defined as follows:

\begin{definition}
	[Semi-faithful order assignment]
	A preorder assignment $\oas$ over $\tris$ is \emph{semi-faithful} if for
	every \SE-interpretation $\tria$ the following conditions are satisfied:
	\begin{enumerate}
		\item For every $\trib \in \tris$ with $\trib \neq \tria$ and $\trib \neq
			\tria^*$, either $\tria \spotrd \trib$ or $\tria^* \spotrd \trib$.

		\item If $\tria^* \potrd \tria$, then $\tria \potrd \tria^*$.
	\end{enumerate}
\end{definition}

Finally, we require the preorder assignment to satisfy one further condition,
related to the well-definedness of sets of \SE-models of every program. It can
be seen as the natural semantic counterpart of \puse{7}.

\begin{definition}
	[Organised order assignment]
	A preorder assignment $\oas$ is \emph{organised} if for all
	\SE-interpretations $\tria$, $\trib$ and all well-defined sets of
	\SE-interpretations $\stria, \strib$ the following condition is satisfied:
	\begin{align*}
		& \text{If }
		\trib \in \min ( \stria, \potrd )
			\cup \min ( \stria, \potrsd )
		\text{ and }
		\trib \in \min ( \strib, \potrd )
			\cup \min ( \strib, \potrsd ), \\
		& \text{then }
		\trib \in \min ( \stria \cup \strib, \potrd )
			\cup \min (\stria \cup \strib, \potrsd ).
	\end{align*}
\end{definition}

Now we are ready to formulate the main result of this section:

\begin{theorem}
	[Representation theorem for rule updates]
	\label{thm:representation}
	Let $\uopr$ be a rule update operator. The following conditions are
	equivalent:
	\begin{enumerate}[a)]
		\item The operator $\uopr$ satisfies conditions \puse{1} -- \puse{8}.

		\item The operator $\uopr$ is characterised by a semi-faithful and
			organised preorder assignment.

		\item The operator $\uopr$ is characterised by a faithful and organised
			partial order assignment.
	\end{enumerate}
\end{theorem}
\begin{proof}
	See \ref{app:proofs}, page~\pageref{proof:representation}.
\end{proof}

This theorem provides a constructive characterisation of rule update operators
satisfying the defined postulates. It facilitates the analysis of their
properties, both semantic as well as computational. Note also that it implies
that the larger class of \emph{semi-faithful} and organised \emph{preorder}
assignments is equivalent to the smaller class of \emph{faithful} and
organised \emph{partial order} assignments. Furthermore, it offers a strategy
for defining operators satisfying the postulates that can be directly applied
whenever an order assignment is known or can be approximated. This strategy is
also complete in the sense that, up to strong equivalence, all operators
satisfying the postulates can be characterised and distinguished by applying
this strategy.

% \section{Winslett's Update Semantics for Rule Updates}

In what follows, we define a specific update operator based on the ideas
underlying Winslett's update semantics \cite{Keller1985,Winslett1990} defined
Section~\ref{sect:preliminaries}. Similarly as was argued in
\cite{Delgrande2008}, since we are working with well-defined sets of
\SE-interpretations, preference needs to be given to their second component.
Thus, we extend the assignment $\oasa$ to all \SE-interpretations $\tria =
\twiab$, $\trib = \tpl{\twic_1, \twid_1}$, $\tric = \tpl{\twic_2, \twid_2}$ as
follows: $\trib \potra \tric$ if and only if the following conditions are
satisfied:
\begin{enumerate}
	\item $(\twid_1 \div \twib) \subseteq (\twid_2 \div \twib)$;

	\item If $(\twid_1 \div \twib) = (\twid_2 \div \twib)$, then $(\twic_1 \div
		\twia) \setminus \Delta \subseteq (\twic_2 \div \twia) \setminus \Delta$
		where $\Delta = \twid_1 \div \twib$.
\end{enumerate}
Intuitively, first we compare the differences between the second components of
$\trib$ and $\tric$ w.r.t.\ $\tria$. If they are equal, we compare the
differences between the first components of $\trib$ and $\tric$ w.r.t.\
$\tria$, but now ignoring the differences between the second components. A
concrete illustration of these comparisons is presented next:

\begin{example}
	[Assignment $\oasa$ for \SE-interpretations]
	Let the \SE-interpretations $\tria$, $\trib$, $\tric_1$, $\tric_2$,
	$\tric_3$ be as follows:\footnote{%
		For the sake of readability, we omit the usual set notation when listing
		\littleSE-interpretations. For example, instead of $\tpl{ \set{\atma},
		\set{\atma, \atmb} }$ we simply write $\tpl{\atma, \atma\atmb}$.
	}
	\begin{align*}
		\tria &= \tpl{\twia, \twib} = \tpl{\atma, \atma\atmb}
		\enspace,
		&
		\trib &= \tpl{\twic, \twid} = \tpl{\atma, \atma\atmc}
		\enspace,
		\\
		\tric_1 &= \tpl{\twic_1, \twid_1} = \tpl{\atma, \atma\atmc\atmd}
		\enspace,
		&
		\tric_2 &= \tpl{\twic_2, \twid_2} = \tpl{\emptyset, \atma\atmc}
		\enspace,
		&
		\tric_3 &= \tpl{\twic_3, \twid_3} = \tpl{\atma\atmc, \atma\atmc}
		\enspace.
	\end{align*}
	We can see that $(\twid \div \twib) = \set{\atmb, \atmc} \subsetneq
	\set{\atmb, \atmc, \atmd} = (\twid_1 \div \twib)$, so it follows that $\trib
	\potra \tric_1$ holds and it is not the case that $\tric_1 \potra \trib$.
	Thus, $\trib \spotra \tric_1$.

	On the other hand, $(\twid \div \twib) = (\twid_2 \div \twib) = (\twid_3
	\div \twib) = \Delta = \set{\atmb, \atmc}$, so $\trib$ and $\tric_2$ can
	only be distinguished based on the second condition. Furthermore, we have
	$(\twic \div \twia) \setminus \Delta = \emptyset \subsetneq \set{\atma} =
	(\twic_2 \div \twia) \setminus \Delta$. Similarly as before, we obtain
	$\trib \spotra \tric_2$.

	A slightly different case occurs with $\tric_3$ because $(\twic_3 \div
	\twia) \setminus \Delta = \set{\atmc} \setminus \set{\atmb, \atmc} =
	\emptyset$ and it follows that both $\trib \potra \tric_3$ and $\tric_3
	\potra \trib$ hold, despite the fact that $\trib \neq \tric_3$.
\end{example}

% Alternative operator definition:
%
% $\trib \potra \tric$ if and only if the following conditions are satisfied:
% \begin{enumerate}
% 	\item $(\twid_1 \div \twib) \subseteq (\twid_2 \div \twib)$;
% 	\item If $(\twid_1 \div \twib) = (\twid_2 \div \twib)$, then $(\twic_1 \div
% 		\twia) \div \Delta \subseteq (\twic_2 \div \twia) \div \Delta$ where
% 		$\Delta = \twid_1 \div \twib$.
% \end{enumerate}
%
% Intuitively, first we compare the differences between the second components
% of $\trib$ and $\tric$ with respect to $\tria$. If they are equal, we
% compare the differences between the first components of $\trib$ and $\tric$
% with respect to $\tria$, but relatively to the differences between the
% second components.
%
% The main problem with this operator is that it has worse results. For
% example, although {p., ~q.} updated by {<- ~q.} has no stable model with the
% original operator, the alternative one will make {p, q} into a stable model.
% This can be verified by observing that while the original operator allows
% both <p, pq> and <pq, pq> to be added as ``minimally close to <p, p>'', the
% alternative operator only keeps <pq, pq>.

Our following result shows that $\oasa$ indeed satisfies the necessary
conditions to characterise rule update operators satisfying the reformulated
postulates.

\begin{proposition}
	\label{prop:oasa is well-defined, faithful and organised}
	The assignment $\oasa$ is a well-defined, faithful and organised preorder
	assignment.
\end{proposition}
\begin{proof}
	See \ref{app:proofs}, page~\pageref{proof:oasa is well-defined, faithful and
	organised}.
\end{proof}
Furthermore, as a consequence of Theorem~\ref{thm:representation} and
Proposition~\ref{prop:oasa is well-defined, faithful and organised}:

\begin{corollary}
	Every rule update operator characterised by $\oasa$ satisfies conditions
	\puse{1} -- \puse{8}.
\end{corollary}

As regards the computational complexity of query answering for rule update
operators characterised by $\oasa$, it follows the same pattern as query
answering for Winslett's belief update operator (c.f.\ Theorems~\ref{thm:oasw
complexity general} and \ref{thm:oasw complexity Horn}). In the general case,
it resides in the second level of the polynomial hierarchy while for definite
programs it drops to the first level. Formally:

\begin{theorem}
	[Computational complexity of rule updates characterised by $\oasa$]
	\label{thm:oasa complexity general}
	Let $\uopr$ be a rule update operator characterised by $\oasa$. Deciding
	whether $\prga \uopr \prgu \entSE \prgb$ for programs $\prga$, $\prgu$,
	$\prgb$ is $\Pi^\ccP_2$-complete. Hardness holds even if $\prga$ is a set of
	positive facts, $\prgu$ is a non-disjunctive program and $\prgb$ contains a
	single fact from $\prga$.
\end{theorem}
\begin{proof}
	See \ref{app:proofs}, page~\pageref{proof:oasa complexity general}.
\end{proof}
\begin{theorem}
	[Computational complexity of definite rule updates characterised by $\oasa$]
	\label{thm:oasa complexity definite}
	Let $\uopr$ be a rule update operator characterised by $\oasa$. Deciding
	whether $\prga \uopr \prgu \entSE \prgb$ for definite programs $\prga$,
	$\prgu$, $\prgb$ is $\cccoNP$-complete. Hardness holds even if $\prga$ is a
	set of facts and $\prgb$ contains a single fact from $\prga$.
\end{theorem}
\begin{proof}
	See \ref{app:proofs}, page~\pageref{proof:oasa complexity definite}.
\end{proof}

\section{Support in Semantic Rule Updates}

\label{sect:comparison}

In this section we take a closer look at the behaviour of semantic rule
update operators.

One of the benefits of dealing with rule updates on the semantic level is
that semantic properties that are rather difficult to show for syntax-based
update operators are much easier to analyse and prove. For example, one of the
most widespread and counterintuitive side effects of syntactic updates is that
they are sensitive to tautological updates. In case of semantic update
operators, such a behaviour is easily shown to be impossible given that the
operator satisfies \puse{2}.

However, semantic update operators do not always behave the way we expect.
Consider first an example using some update operator $\uopr$
characterised by the order assignment $\oasa$ defined in the previous
section:\footnote{%
	It has been shown that Winslett's update semantics has some drawbacks, just
	as other update operators previously proposed in the context of Classical
	Logic do (see \cite{Herzig1999} for a survey). Nevertheless, we decided to
	choose Winslett's update operator as the basis to define a rule update
	operator and illustrate its properties because it is one of the most
	extensively studied and understood update operators, and because the
	undesired behaviour illustrated in this example is shared by all update
	operators based on KM postulates and \littleSE-models -- as we shall see --
	and not a specific problem due to our choice of Winslett's operator.
}

\begin{example}
	\label{ex:support and fact update vs PU4}
	Let the programs $\prga$, $\prgb$ and $\prgu$ be as follows:
	\begin{align*}
		\prga:\quad \atma&.
		&
		\prgb:\quad \atma &\lpif \atmb.
		&
		\prgu:\quad \lpnot \atmb. \\
		\atmb&.
		&
		\atmb&.
	\end{align*}
	It can be easily verified that:% \footnote{%
%		For the sake of readability, we omit the usual set notation when listing
%		\littleSE-interpretations. For example, instead of $\tpl{ \set{\atma},
%		\set{\atma, \atmb} }$ we simply write $\tpl{\atma, \atma\atmb}$.
%	}
	\begin{align*}
		\Modse{\prga \uopr \prgu}
		=
		\Modse{\prgb \uopr \prgu}
		=
		\Set{ \tpl{\atma, \atma} }
		\enspace.
	\end{align*}
	Hence, both $\prga \uopr \prgu$ and $\prgb \uopr \prgu$ have the single
	answer set $\twib = \set{\atma}$. In case of $\prga \uopr \prgu$ this is
	indeed the expected result. But in case of $\prgb \uopr \prgu$ we can see
	that $\atma$ is true in $\twib$ even though there is no rule in $\prgb \cup
	\prgu$ justifying it, i.e.\ there is no rule with $\atma$ in its head and
	its body satisfied in $\twib$. This means that the behaviour of $\uopr$ is
	in discord with intuitions underlying most Logic Programming semantics.
\end{example}

In the following we show that such counterintuitive behaviour is not specific
to $\uopr$, but extends to all semantic update operators for answer-set
programs based on the well-established notions of \SE-models and KM
postulates. This is especially interesting from the point of view of
comparison with syntax-based approaches to rule updates that, as we formally
pinpoint in what follows, do not suffer from such drawbacks.

The property of \emph{support} \cite{Apt1988,Dix1995a} is one of the basic
conditions that Logic Programming semantics are intuitively designed to
satisfy. In the static case, this property can be formulated as follows:

\begin{definition}
	[Static support]
	Let $\prg$ be a program, $\atm$ an atom and $\twib$ an interpretation. We
	say that \emph{$\prg$ supports $\atm$ in $\twib$} if there is some rule $\rl
	\in \prg$ such that $\atm \in \hrl$ and $\twib \ent \brl$.

	A Logic Programming semantics \S{} is \emph{supported} if for each model
	$\twib$ of a program $\prg$ under \S{} the following condition is
	satisfied: Every atom $\atm \in \twib$ is supported by $\prg$ in $\twib$.
\end{definition}

A supported semantics thus requires all atoms in an assigned model to be in
the head of some rule with a satisfied body, ensuring that no atom is true
without at least \emph{some} justification. Note that the widely accepted
Logic Programming semantics, such as the answer-set and well-founded
semantics, are supported (see \cite{Dix1995,Dix1995a} for more on properties
of Logic Programming semantics).

It is only natural to require that rule update operators do not neglect this
essential property which also gives rise to much of the intuitive appeal of
Logic Programming systems. As it turns out, it is not difficult to verify that
despite the substantial differences between various syntax-based approaches to
rule updates and revision, all of the semantics introduced in
\cite{Leite1997,Alferes2000,Eiter2002,Sakama2003,Alferes2005,Zhang2006,Delgrande2007,Delgrande2010}
respect support in the following sense:

\begin{definition}
	[Dynamic support]
	\label{def:dynamic support}
	We say that a rule update operator $\uopr$ \emph{respects support} if the
	following condition is satisfied for all programs $\prga$, $\prgu$ and all
	answer sets $\twib$ of $\prga \uopr \prgu$: Every atom $\atm \in \twib$ is
	supported by $\prga \cup \prgu$ in $\twib$.
\end{definition}

So an update operator respects support if it returns only programs whose
answer sets are supported by rules from either the original program or from
its update. Similarly as in the case of static support, this amounts to the
requirement that an atom may be true only if at least \emph{some}
justification can be found for it. 

Another basic expectation from an update operator is the usual intuition
regarding how facts should be updated by newer facts. It enforces the
principle of literal inertia, but only for the case when both the initial
program and its update are consistent sets of facts. Similarly as with
support, a variety of different syntax-based approaches to rule updates and
revision, in particular the semantics introduced in
\cite{Leite1997,Alferes2000,Eiter2002,Sakama2003,Alferes2005,Zhang2006,Delgrande2007,Delgrande2010},
satisfy fact update in the following sense:

\begin{definition}
	[Fact update]
	\label{def:fact update}
	We say that a rule update operator $\uopr$ \emph{respects fact update} if
	for all consistent sets of facts $\prga$, $\prgu$, the unique answer set of
	$\prga \uopr \prgu$ is the interpretation
	\[
		\Set{
			\atm
			|
			(\atm.) \in \prga \cup \prgu
				\land (\lpnot \atm.) \notin \prgu
		} \enspace.
	\]
\end{definition}

Thus, a rule update operator respects fact update if it is well-behaved
w.r.t.\ consistent sets of facts: it provides the answer set that contains
exactly those atoms that are asserted as true in either the original program
or its update, and are not asserted as false in the update. This behaviour is
widely accepted -- it stems from the intuitions regarding database updates and
is uncontroversial in both the belief change and rule change communities.

We conjecture that any reasonable update operator for answer-set programs
should satisfy support and fact update since these two properties place basic
constraints on its behaviour and are based on fundamental and widely accepted
intuitions. They are by no means exhaustive or sufficient -- it is not
difficult to define rule update operators that satisfy both of them but are
sensitive to tautological updates or quickly end up in an inconsistent state
without a possibility of recovery -- but they both seem necessary, even
elementary, properties of a well-behaved rule update operator. However, it
turns out that every rule update operator based on \SE-models, even if it
satisfies only the basic postulate that enforces syntax independence, fails to
comply with at least one of these two basic expectations.

\begin{theorem}
	\label{thm:impossibility}
	A rule update operator that satisfies \puse{4} either does not respect
	support or it does not respect fact update.
\end{theorem}
\begin{proof}
	Let $\uopr$ be a rule update operator that satisfies \puse{4} and
	consider again the programs $\prga$, $\prgb$ and $\prgu$ from
	Example~\ref{ex:support and fact update vs PU4}. Since $\prga$ is strongly
	equivalent to $\prgb$, by \puse{4} we obtain that $\prga \uopr \prgu$ is
	strongly equivalent to $\prgb \uopr \prgu$. Consequently, $\prga \uopr
	\prgu$ has the same answer sets as $\prgb \uopr \prgu$. It only remains to
	observe that if $\uopr$ respects fact update, then $\prga \uopr \prgu$ has
	the unique answer set $\set{\atm}$. But then $\set{\atm}$ is an answer set
	of $\prgb \uopr \prgu$ in which $\atm$ is unsupported by $\prgb \cup \prgu$.
	Hence $\uopr$ does not respect support.
\end{proof}

So any answer-set program update operator based on \SE-models and the KM
approach to belief update, as materialised in the fundamental principle
\puse{4}, cannot respect two basic and desirable properties: support and fact
update. We believe that this is a major drawback of such operators, severely
diminishing their applicability.

Moreover, the principle \puse{4} is also adopted for \emph{revision} of
answer-set programs based on \SE-models in \cite{Delgrande2008}.\footnote{%
	Note that the belief update postulate \smallbu{4}, from which \smallpuse{4}
	originates, is also one of the reformulated AGM postulates for belief
	\emph{revision} \cite{Katsuno1992}. The original AGM framework
	\cite{Alchourron1985} assumes that the initial knowledge base $\kb$ is
	closed w.r.t.\ logical consequence and the first AGM postulate requires that
	the result of revision also be a closed set. Under these assumptions,
	different knowledge bases cannot be equivalent and, as a consequence, the
	original AGM postulate corresponding to \smallbu{4} is
	\pstl[\scriptsize]{$\uoprev$5}: If $\cn(\frmu) = \cn(\frmv)$, then $\kb
	\uoprev \frmu = \kb \uoprev \frmv$ (where $\cn$ is the logical consequence
	operator and $\uoprev$ the revision operator.)
}
This means that Theorem~\ref{thm:impossibility} extends to semantic program
revision operators, such as those defined in \cite{Delgrande2008}: Whenever
support and fact update are expected to be satisfied by a rule revision
operator, it cannot be defined by purely manipulating the sets of \SE-models
of the underlying programs.

% \section{Properties for Rule Change Based on \SE-Models}

% \label{sect:properties based on se-models}

One question that suggests itself is whether a weaker version of the principle
\puse{4} can be combined with properties such as support and fact update. Its
two immediate weakenings, analogous to the weakenings of \bu{4} in
\cite{Herzig1999}, are as follows:
\begin{enumerate}
	\renewcommand{\theenumi}{\labpuse{P4.\arabic{enumi}}}
	\renewcommand{\labelenumi}{\theenumi\hfill}
	\setlength{\itemsep}{.5ex}

	\item If $\prga \eqSE \prgb$, then $\prga \uopr \prgu \eqSE \prgb \uopr
		\prgu$.
		\label{pstl:puse:4.1}

	\item If $\prgu \eqSE \prgv$, then $\prga \uopr \prgu \eqSE \prga \uopr
		\prgv$.
		\label{pstl:puse:4.2}
\end{enumerate}
In case of \puse{4.1}, it is easy to see that the proof of
Theorem~\ref{thm:impossibility} applies in the same way as with \puse{4}, so
\puse{4.1} is likewise incompatible with support and fact update.

On the other hand, principle \puse{4.2}, also referred to as Weak Independence
of Syntax (WIS) \cite{Osorio2007}, does not suffer from such severe
limitations. It is, nevertheless, violated by syntax-based rule update
semantics that assign a special meaning to occurrences of default literals in
heads of rules, as illustrated in the following example:
\begin{example}
	Let the programs $\prga$, $\prgu$ and $\prgv$ be as follows:
	\begin{align*}
		\prga:\quad \atma&.
		&
		\prgu:\quad \lpnot \atma &\lpif \atmb.
		&
		\prgv:\quad \lpnot \atmb &\lpif \atma. \\
		\atmb&.
	\end{align*}
	Since $\prgu$ is strongly equivalent to $\prgv$, \puse{4.2} requires that
	$\prga \uopr \prgu$ be strongly equivalent to $\prga \uopr \prgv$. This is
	in contrast with the rule update semantics of
	\cite{Leite1997,Alferes2000,Alferes2005} where a default literal $\lpnot
	\atma$ in the head of a rule indicates that whenever the body of the rule is
	satisfied, there is a \emph{reason for $\atma$ to cease being true}. A
	consequence of this is that an update of $\prga$ by $\prgu$ results in the
	single answer set $\set{\atmb}$ while an update by $\prgv$ leads to the
	single answer set $\set{\atma}$.
\end{example}
Thus, when considering the principle \puse{4.2}, benefits of the
declarativeness that it brings with it need to be weighed against the loss of
control over the results of updates by rules with default literals in their
heads.

The problems we identified might be mitigated if a richer semantic
characterisation of logic programs was used instead of \SE-models. Such a
characterisation would have to be able to distinguish between programs such as
$\prga = \set{\atma., \atmb.}$ and $\prgb = \set{\atma \lpif \atmb., \atmb.}$
because they are expected to behave differently when subject to evolution.

Another alternative is to use one of the syntactic approaches to rule
updates, e.g.\ \cite{Alferes2005}, that have matured over the years.

\section{Conclusion}

\label{sect:conclusion}

In this paper we revisited the problem of updates of answer-set programs, in
an attempt to change the focus from the syntactic representation of a program
to its semantic content and to facilitate the analysis of semantic properties
of defined update operators. We did so by applying the established approach to
updates following Katsuno and Mendelzon's postulates in the context of logic
programs. Whereas until recently this was not possible since these postulates
were simply not applicable (nor adaptable) when considering non-monotonic
Logic Programming semantics, as shown in \cite{Eiter2002}, the introduction of
\emph{\SE-models} \cite{Turner2003}, which provide a monotonic
characterisation of logic programs that is strictly more expressive than the
answer-set semantics, provided a new opportunity to cast KM postulates into
Logic Programming.

We adapted the KM postulates to be used for answer-set program updates and
showed a representation theorem which provides a constructive characterisation
of rule update operators satisfying the postulates. This characterisation
not only facilitates the investigation of these operators' properties, both
semantic as well as computational, but it also provides an intuitive strategy
for constructively defining these operators. This is one of the major
contributions of the paper since it brings, for the first time, updates of
answer-set programs in line with KM postulates. We illustrated this result
with a definition of a specific rule update operator which is a counterpart
of Winslett's belief update operator.

The second important contribution of this paper is the uncovering of a serious
drawback that extends to all answer-set program update operators based on
\SE-models and AGM-style approach to program revision and update. All such
operators violate at least one of two basic and very desirable properties. The
first one consists of respecting \emph{support}, a property that is enjoyed,
in the static case, by all widely accepted Logic Programming semantics. The
second property, \emph{fact update}, is concerned with the answer set assigned
to a consistent set of facts after it is updated by another consistent set of
facts. This contribution is very important as it should guide further research
on updates of answer-set programs
\begin{enumerate}[a)]
	\item away from the purely semantic approach materialised in AGM and KM
		postulates, or

	\item to the development of semantic characterisations of answer-set
		programs that are richer than \SE-models and appropriately capture their
		dynamic behaviour, such as in \cite{Slota2012}, or even

	\item turning back to the more syntactic approaches, such as
		\cite{Alferes2005}, and see whether they indeed offer a viable
		alternative.
\end{enumerate}
Either way, updating answer-set programs is a very important theoretical and
practical problem that is still waiting for a definite solution. Also, despite
the issues with the syntax independence postulate \puse{4}, other principles
based on \SE-models play an important role with regards to the classification
and evaluation of different approaches to rule change. For instance, the
reformulations of rule change principles from \cite{Eiter2002} in terms of
strong equivalence, considered already in \cite{Delgrande2008}, can be
formulated as follows:
\begin{enumerate}
	\renewcommand{\theenumi}{\labpuse{Initialisation}}
	\renewcommand{\labelenumi}{\hfill\theenumi}
	\setlength{\itemsep}{.5ex}
	\setlength{\labelwidth}{7em}
	\setlength{\itemindent}{5em}

	\item $\emptyset \uopr \prgu \eqSE \prgu$.
		\label{pstl:puse:initialisation}

	\renewcommand{\theenumi}{\labpuse{Idempotence}}
	\item $\prga \uopr \prga \eqSE \prga$.
		\label{pstl:puse:idempotence}

	\renewcommand{\theenumi}{\labpuse{Tautology}}
	\item If $\prgu \eqSE \emptyset$, then $\prga \uopr \prgu \eqSE \prga$.
		\label{pstl:puse:tautology}

	\renewcommand{\theenumi}{\labpuse{Absorption}}
	\item If $\prgu \eqSE \prgv$, then $(\prga \uopr \prgu) \uopr \prgv \eqSE
		\prga \uopr \prgv$.
		\label{pstl:puse:absorption}

	\renewcommand{\theenumi}{\labpuse{Augmentation}}
	\item If $\prgv \entSE \prgu$, then $(\prga \uopr \prgu) \uopr \prgv \eqSE
		\prga \uopr \prgv$.
		\label{pstl:puse:augmentation}
\end{enumerate}
We believe that all of these properties are indeed desirable and strengthen
their original formulations in an interesting way. Investigation of operators
with these properties, as well as a further analysis of the postulates
\puse{1} -- \puse{8}, remains an important research topic. This paper
contains, we believe, a relevant contribution to a better understanding of
rule change that will help guide future research.

% \puse{initialisation}
% \puse{idempotence}
% \puse{tautology}
% \puse{absorption}
% \puse{augmentation} 

\section*{Acknowledgement}

We would like to thank the anonymous reviewers for their valuable comments.
M.\ Slota was supported by FCT scholarship SFRH/BD/38214/2007. J.\ Leite was
partially supported by FCT funded project ERRO (PTDC/EIA-CCO/121823/2010).

\appendix

\section{Proofs: Representation Theorem}

\label{app:proofs}

\begin{definition}
	[Program corresponding to a set of \SE-interpretations]
	Let $\stri$ be a set of \SE-interpretations. We denote by $\synt{\stri}$
	some arbitrary but fixed program $\prg$ such that
	\[
		\modse{\prg}
		=
		\Set{\tria, \tria^* | \tria \in \stri}
		\enspace.
	\]
	Instead of $\synt{\set{\tria_1, \tria_2, \dotsc, \tria_\lng}}$ we usually
	write $\synt{\tria_1, \tria_2, \dotsc, \tria_\lng}$.
\end{definition}

\begin{definition}
	[Order assignment generated by an update operator]
	\label{def:preorder assignment generated by operator}
	Let $\uopr$ be a rule update operator and $\tria$ an \SE-interpretation.
	We define the binary relation $\prec^\tria_{\uopr}$ for all
	\SE-interpretations $\trib$, $\tric$ as follows: $\trib \prec^\tria_{\uopr}
	\tric$ if and only if the following conditions are satisfied:
	\begin{align}
		& \trib \in \modse{\synt{\tria} \uopr \synt{\trib, \tric}}
			\label{eq:preorder assignment generated by operator:1} \\
		& \tric \notin \modse{\synt{\tria} \uopr \synt{\trib, \tric}}
			\label{eq:preorder assignment generated by operator:2} \\
		& \text{If } \trib \neq \trib^* \text{, then }
			\tric \in \modse{\synt{\tria} \uopr \synt{\trib^*, \tric}}
			\label{eq:preorder assignment generated by operator:3}
	\end{align}
	The \emph{preorder assignment generated by $\uopr$} assigns to every
	\SE-interpretation $\tria$ the reflexive and transitive closure $\potro$ of
	$\prec^\tria_{\uopr}$, i.e.\ $\trib \potro \tric$ if and only if $\trib =
	\tric$ or there is some $\lng \geq 2$ and \SE-interpretations $\trib_1,
	\trib_2, \dotsc, \trib_\lng$ such that $\trib = \trib_1 \prec^\tria_{\uopr}
	\trib_2 \prec^\tria_{\uopr} \dotsb \prec^\tria_{\uopr} \trib_\lng = \tric$.
\end{definition}

\begin{lemma}
	\label{lemma:if less then not model}
	Let $\uopr$ be a rule update operator satisfying conditions \puse{1} --
	\puse{8} and $\tria$, $\trib$, $\tric$ some \SE-interpretations. If $\trib
	\potro \tric$, then either $\trib = \tric$ or $\tric \notin
	\modse{\synt{\tria} \uopr \synt{\trib, \tric}}$.
\end{lemma}
\begin{proof*}
	Suppose that $\trib \neq \tric$. Then, by the definition of $\potro$, for
	some $\lng \leq 2$ and \SE-interpretations $\trib_1, \trib_2, \dotsc,
	\trib_\lng$ it holds that $\trib = \trib_1 \prec^\tria_{\uopr} \trib_2
	\prec^\tria_{\uopr} \dotsb \prec^\tria_{\uopr} \trib_\lng = \tric$. We will
	prove by induction on $\lng$ that $\trib_\lng \notin \modse{\synt{\tria}
	\uopr \synt{\trib_1, \trib_\lng}}$ from which the desired result follows
	directly.

	\begin{enumerate}[1$^\circ$]
		\item For $\lng = 2$ this follows from $\trib_1 \prec^\tria_{\uopr} \trib_2$
			by \eqref{eq:preorder assignment generated by operator:2}.

		\item We inductively assume that
			\begin{equation}
				\label{eq:proof:if less then not model:1}
				\trib_\lng \notin \modse{\synt{\tria} \uopr \synt{\trib_1, \trib_\lng}}
			\end{equation}
			and prove that $\trib_{\lng+1} \notin \modse{\synt{\tria} \uopr
			\synt{\trib_1, \trib_{\lng+1}}}$.

			We know that $\trib_\lng \prec^\tria_{\uopr} \trib_{\lng+1}$, so by
			\eqref{eq:preorder assignment generated by operator:2} we obtain
			\begin{equation}
				\label{eq:proof:if less then not model:2}
				\trib_{\lng+1}
				\notin
				\modse{\synt{\tria} \uopr \synt{\trib_\lng, \trib_{\lng+1}}}
				\enspace.
			\end{equation}
			Considering that the program $\synt{\trib_1, \trib_\lng, \trib_{\lng+1}}
			\prgand \synt{\trib_1, \trib_\lng}$ is strongly equivalent to
			$\synt{\trib_1, \trib_\lng}$, by \puse{5} and \puse{4} we conclude that
			\[
				(\synt{\tria} \uopr \synt{\trib_1, \trib_\lng, \trib_{\lng+1}})
				\prgand
				\synt{\trib_1, \trib_\lng}
				\entSE
				\synt{\tria} \uopr \synt{\trib_1, \trib_\lng}
			\]
			which, together with \eqref{eq:proof:if less then not model:1}, implies
			that
			\begin{equation}
				\label{eq:proof:if less then not model:3}
				\trib_\lng
				\notin
				\modse{\synt{\tria} \uopr \synt{\trib_1, \trib_\lng, \trib_{\lng+1}}}
				\enspace.
			\end{equation}

			Similarly, since the program $\synt{\trib_1, \trib_\lng, \trib_{\lng+1}}
			\prgand \synt{\trib_\lng, \trib_{\lng+1}}$ is strongly equivalent to
			$\synt{\trib_\lng, \trib_{\lng+1}}$, by \puse{5} and \puse{4} we obtain
			that
			\[
				(\synt{\tria} \uopr \synt{\trib_1, \trib_\lng, \trib_{\lng+1}})
				\prgand
				\synt{\trib_\lng, \trib_{\lng+1}}
				\entSE
				\synt{\tria} \uopr \synt{\trib_\lng, \trib_{\lng+1}}
				\enspace,
			\]
			and so due to \eqref{eq:proof:if less then not model:2} it holds that
			\begin{equation}
				\label{eq:proof:if less then not model:4}
				\trib_{\lng+1}
				\notin
				\modse{\synt{\tria} \uopr \synt{\trib_1, \trib_\lng, \trib_{\lng+1}}}
				\enspace.
			\end{equation}

			Now we consider two cases:
			\begin{enumerate}[a)]
				\item If $\trib_\lng = \trib_\lng^*$, then \eqref{eq:proof:if less
					then not model:3} and \puse{1} imply that
					\begin{align*}
						\synt{\tria} \uopr \synt{\trib_1, \trib_\lng, \trib_{\lng+1}}
						&\entSE
						\synt{\trib_1, \trib_{\lng+1}}
						\enspace; \\
						\synt{\tria} \uopr \synt{\trib_1, \trib_{\lng+1}}
						&\entSE
						\synt{\trib_1, \trib_\lng, \trib_{\lng+1}}
						\enspace,
					\end{align*}
					so by \puse{6} we can conclude that $\synt{\tria} \uopr
					\synt{\trib_1, \trib_\lng, \trib_{\lng+1}}$ is strongly equivalent
					to $\synt{\tria} \uopr \synt{\trib_1, \trib_{\lng+1}}$. But then the
					desired conclusion follows from \eqref{eq:proof:if less then not
					model:4}.

				\item If $\trib_\lng \neq \trib_\lng^*$, then from \eqref{eq:preorder
					assignment generated by operator:3} we infer that
					\begin{equation}
						\label{eq:proof:if less then not model:5}
						\trib_{\lng+1}
						\in
						\modse{\synt{\tria} \uopr \synt{\trib_\lng^*, \trib_{\lng+1}}}
						\enspace.
					\end{equation}
					Furthermore, from \eqref{eq:proof:if less then not model:3} and
					\puse{1} we obtain
					\begin{align*}
						\synt{\tria} \uopr \synt{\trib_1, \trib_\lng, \trib_{\lng+1}}
						&\entSE
						\synt{\trib_1, \trib_\lng^*, \trib_{\lng+1}}
						\enspace; \\
						\synt{\tria} \uopr \synt{\trib_1, \trib_\lng^*, \trib_{\lng+1}}
						&\entSE
						\synt{\trib_1, \trib_\lng, \trib_{\lng+1}}
						\enspace,
					\end{align*}
					so by \puse{6} we can conclude that $\synt{\tria} \uopr
					\synt{\trib_1, \trib_\lng, \trib_{\lng+1}}$ is strongly equivalent
					to $\synt{\tria} \uopr \synt{\trib_1, \trib_\lng^*, \trib_{\lng+1}}$
					and, due to \eqref{eq:proof:if less then not model:4},
					\[
						\trib_{\lng + 1}
						\notin
						\modse{
							\synt{\tria}
							\uopr
							\synt{\trib_1, \trib_\lng^*, \trib_{\lng+1}}
						}
						\enspace.
					\]
					Since $\synt{\trib_1, \trib_\lng^*, \trib_{\lng+1}}$ is strongly
					equivalent to $\synt{\trib_1, \trib_{\lng+1}} \prgor
					\synt{\trib_\lng^*, \trib_{\lng+1}}$, it follows from \puse{4} and
					\puse{7} that either $\trib_{\lng+1} \notin \modse{\synt{\tria}
					\uopr \synt{\trib_1, \trib_{\lng+1}}}$ or $\trib_{\lng+1} \notin
					\modse{\synt{\tria} \uopr \synt{\trib_\lng^*, \trib_{\lng+1}}}$. The
					latter is impossible due to \eqref{eq:proof:if less then not
					model:5}. \qedhere
			\end{enumerate}
	\end{enumerate}
\end{proof*}

\begin{lemma}
	\label{lemma:if not model then less}
	Let $\uopr$ be a rule update operator satisfying conditions \puse{1} --
	\puse{8} and $\tria$, $\trib$, $\tric$, some \SE-interpretations. If $\trib
	\nless^\tria_{\uopr} \tric$, then the following conditions are satisfied:
	\begin{enumerate}[(1)]
		\item If $\trib = \tric^*$, then $\tric \in \modse{\synt{\tria} \uopr
			\synt{\trib, \tric}}$.

		\item If $\trib = \trib^*$ and $\tric \in \modse{\synt{\tria} \uopr
			\synt{\tric}}$, then $\tric \in \modse{\synt{\tria} \uopr \synt{\trib,
			\tric}}$.

		\item If $\trib \neq \trib^*$ and $\tric \in \modse{\synt{\tria} \uopr
			\synt{\trib^*, \tric}}$, then $\tric \in \modse{\synt{\tria} \uopr
			\synt{\trib, \tric}}$.
	\end{enumerate}
\end{lemma}
\begin{proof}
	First we show the following auxiliary statement: If $\trib = \tric$ or
	$\trib \notin \modse{\synt{\tria} \uopr \synt{\trib, \tric}}$, then all
	three conditions are satisfied.

	First suppose that $\trib = \tric$. If $\trib = \tric^*$, then $\trib =
	\trib^* = \tric = \tric^*$, so it follows from \puse{1} and \puse{3} that
	$\modse{\synt{\tria} \uopr \synt{\trib, \tric}} = \modse{\synt{\tria} \uopr
	\synt{\tric^*}} = \set{\tric^*}$, verifying condition (1). Furthermore,
	conditions (2) and (3) are satisfied because $\synt{\tric} = \synt{\trib^*,
	\tric} = \synt{\trib, \tric}$.

	Now suppose that $\trib \notin \modse{\synt{\tria} \uopr \synt{\trib,
	\tric}}$. If $\trib = \tric^*$, then it follows from \puse{1} and \puse{3}
	that $\tric \in \modse{\synt{\tria} \uopr \synt{\trib, \tric}}$. If $\trib =
	\trib^*$, then it follows from \puse{1} that
	\begin{align*}
		\synt{\tria} \uopr \synt{\trib, \tric}
		&\entSE
		\synt{\tric}
		& \text{and} &
		& \synt{\tria} \uopr \synt{\tric}
		&\entSE
		\synt{\trib, \tric}
		\enspace,
	\end{align*}
	so by \puse{6} we obtain that $\synt{\tria} \uopr \synt{\trib, \tric} \eqSE
	\synt{\tria} \uopr \synt{\tric}$. Hence, it follows from $\tric \in
	\modse{\synt{\tria} \uopr \synt{\tric}}$ that $\tric \in \modse{\synt{\tria}
	\uopr \synt{\trib, \tric}}$. On the other hand, if $\trib \neq \trib^*$,
	then it follows from \puse{1} that
	\begin{align*}
		\synt{\tria} \uopr \synt{\trib, \tric}
		&\entSE
		\synt{\trib^*, \tric}
		& \text{and} &
		& \synt{\tria} \uopr \synt{\trib^*, \tric}
		&\entSE
		\synt{\trib, \tric}
		\enspace,
	\end{align*}
	so by \puse{6} we obtain that $\synt{\tria} \uopr \synt{\trib, \tric} \eqSE
	\synt{\tria} \uopr \synt{\trib^*, \tric}$. Hence it follows from $\tric \in
	\modse{\synt{\tria} \uopr \synt{\trib^*, \tric}}$ that $\tric \in
	\modse{\synt{\tria} \uopr \synt{\trib, \tric}}$.

	Turning to the proof of the lemma, note that since $\trib \nless^\tria_{\uopr}
	\tric$, either $\trib \nleq^\tria_{\uopr} \tric$ or $\tric \potro \trib$. In
	the former case, $\trib \nprec^\tria_{\uopr} \tric$, so, by the definition of
	$\prec^\tria_{\uopr}$, either $\trib \notin \modse{\synt{\tria} \uopr
	\synt{\trib, \tric}}$, so we can apply our auxiliary statement, or $\tric
	\in \modse{\synt{\tria} \uopr \synt{\trib, \tric}}$ as desired, or $\trib
	\neq \trib^*$ and $\tric \notin \modse{\synt{\tria} \uopr \synt{\trib^*,
	\tric}}$, in which case all three conditions are trivially satisfied. In the
	latter case it follows from Lemma~\ref{lemma:if less then not model} that
	either $\trib = \tric$ or $\trib \notin \modse{\synt{\tria} \uopr
	\synt{\trib, \tric}}$, so the rest follows once again from the auxiliary
	statement.
\end{proof}

\begin{proposition}
	\label{prop:characterisation for basic programs}
	Let $\uopr$ be a rule update operator satisfying conditions \puse{1} --
	\puse{8}, $\tria$ an \SE-interpretation and $\prgu$ a program. Then,
	\[
		\modse{\synt{\tria} \uopr \prgu}
		=
		\min \left( \modse{\prgu}, \potro \right)
		\enspace.
	\]
\end{proposition}
\begin{proof}
	First take some $\tric \in \modse{\synt{\tria} \uopr \prgu}$. By \puse{1},
	$\tric \in \modse{\prgu}$. Suppose that $\tric$ is not minimal in
	$\modse{\prgu}$ w.r.t.\ $\potro$. Then there is some $\trib \in
	\modse{\prgu}$ such that $\trib \spotro \tric$. Thus, $\trib \neq \tric$,
	and by Lemma~\ref{lemma:if less then not model} we conclude that $\tric
	\notin \modse{\synt{\tria} \uopr \synt{\trib, \tric}}$. Considering that
	$\prgu \prgand \synt{\trib, \tric}$ is strongly equivalent to $\synt{\trib,
	\tric}$, it follows from \puse{4} and \puse{5} that $(\synt{\tria} \uopr
	\prgu) \prgand \synt{\trib, \tric} \entSE \synt{\tria} \uopr \synt{\trib,
	\tric}$. Consequently, $\tric \notin \modse{\synt{\tria} \uopr \prgu}$,
	contrary to our assumption.  Therefore, $\modse{\synt{\tria} \uopr \prgu}$
	is a subset of $\min( \modse{\prgu}, \potro)$.

	To prove the converse inclusion, assume that $\tric$ is minimal in
	$\modse{\prgu}$ w.r.t.\ $\potro$ and take some $\trib \in \modse{\prgu}$.
	Note that $\trib \nless^\tria_{\uopr} \tric$, so we can use
	Lemma~\ref{lemma:if not model then less}. We will show that $\tric \in
	\modse{\synt{\tria} \uopr \synt{\trib, \tric}}$. We consider three cases:
	\begin{enumerate}[a)]
		\item If $\trib = \tric^*$, then $\tric \in \modse{\synt{\tria} \uopr
			\synt{\trib, \tric}}$ follows immediately from condition (1) of
			Lemma~\ref{lemma:if not model then less}.

		\item If $\trib = \trib^*$, then the previous case together with the fact
			that $\modse{\prgu}$ is well-defined entails that $\tric \in
			\modse{\synt{\tria} \uopr \synt{\tric}}$ and by condition (2) of
			Lemma~\ref{lemma:if not model then less} it follows that $\tric \in
			\modse{\synt{\tria} \uopr \synt{\trib, \tric}}$.

		\item If $\trib \neq \trib^*$, then the previous case together with the
			fact that $\modse{\prgu}$ is well-defined entails that $\tric \in
			\modse{\synt{\tria} \uopr \synt{\trib^*, \tric}}$ and by condition (3)
			of Lemma~\ref{lemma:if not model then less} it follows that $\tric \in
			\modse{\synt{\tria} \uopr \synt{\trib, \tric}}$.
	\end{enumerate}
	The choice of $\trib$ was arbitrary, so we have proven that $\tric \in
	\modse{\synt{\tria} \uopr \synt{\trib, \tric}}$ for all $\trib \in
	\modse{\prgu}$. This means that by repeated application of \puse{7}, $\tric$
	is an \SE-model of the program
	\[
		\synt{\tria}
		\uopr
		\bigprgor_{\trib \in \modse{\prgu}} \synt{\trib, \tric}
	\]
	and since $\prgu$ is strongly equivalent to the program $\bigprgor_{\trib
	\in \modse{\prgu}} \synt{\trib, \tric}$, it follows from \puse{4} that
	$\tric \in \modse{\synt{\tria} \uopr \prgu}$.
\end{proof}

\begin{proposition}
	\label{prop:if postulates then ordering}
	If a rule update operator $\uopr$ satisfies conditions \puse{1} --
	\puse{8}, then the preorder assignment generated by $\uopr$ is semi-faithful
	and organised and it characterises $\uopr$.
\end{proposition}
\begin{proof}
	First we show that the assignment generated by $\uopr$ characterises
	$\uopr$. We know that $\prga$ is strongly equivalent to the program
	$\bigprgor_{\tria \in \modse{\prga}} \synt{\tria}$, so by \puse{4} and
	repeated application of \puse{8} we obtain that $\prga \uopr \prgu$ is
	strongly equivalent to the program
	\[
		\bigprgor_{\tria \in \modse{\prga}} (\synt{\tria} \uopr \prgu)
		\enspace.
	\]
	Furthermore, Proposition~\ref{prop:characterisation for basic programs}
	implies that $\modse{\synt{\tria} \uopr \prgu} = \min \left( \modse{\prgu},
	\potro \right)$, so indeed
	\begin{equation}
		\label{eq:proof:prop:if postulates then ordering}
		\modse{\prga \uopr \prgu}
		=
		\bigcup_{\tria \in \modse{\prga}} \modse{\synt{\tria} \uopr \prgu}
		=
		\bigcup_{\tria \in \modse{\prga}}
		\min \left( \modse{\prgu}, \potro \right)
		\enspace.
	\end{equation}

	To see that the assignment generated by $\uopr$ is semi-faithful, first take
	some \SE-interpretations $\tria$, $\trib$ such that $\trib \neq \tria$ and
	$\trib \neq \tria^*$. We need to show that either $\tria \spotro \trib$ or
	$\tria^* \spotro \trib$. The equation \eqref{eq:proof:prop:if postulates
	then ordering} together with \puse{2} imply that
	\begin{align*}
		\modse{\synt{\tria} \uopr \synt{\trib^*, \tria}}
			&=
			\min \left( \set{\trib^*, \tria, \tria^*}, \potro \right)
			\cup
			\min \left( \set{\trib^*, \tria, \tria^*}, \potrso \right) \\
			&= \set{\tria, \tria^*}
			\enspace, \\
		\modse{\synt{\tria} \uopr \synt{\trib, \tria}}
			&=
			\min \left(
				\set{\trib, \trib^*, \tria, \tria^*},
				\potro
			\right)
			\cup
			\min \left(
				\set{\trib, \trib^*, \tria, \tria^*},
				\potrso
			\right) \\
			&= \set{\tria, \tria^*}
			\enspace.
	\end{align*}
	Thus, $\trib^*$ is not minimal within $\set{\trib^*, \tria, \tria^*}$ and
	$\trib$ is not minimal within $\set{\trib, \trib^*, \tria, \tria^*}$ w.r.t.\
	$\potro$. In other words:
	\begin{align}
		& \text{either }
			\tria \spotro \trib^* \text{ or } \tria^* \spotro \trib^*
			\text{ and}
			\label{eq:proof:semi-faithful:1} \\
		& \text{either }
			\tria \spotro \trib \text{ or } \tria^* \spotro \trib
			\text{ or }
			\trib^* \spotro \trib
			\label{eq:proof:semi-faithful:2} \enspace.
	\end{align}
	In case of the first two alternatives of \eqref{eq:proof:semi-faithful:2},
	we have already achieved our goal. The third alternative together with
	\eqref{eq:proof:semi-faithful:1} and transitivity of $\spotro$ also
	concludes the proof of the first condition of semi-faithfulness. To see that
	the second condition holds as well, consider that by \puse{2},
	$\modse{\synt{\tria^*} \uopr \synt{\tria}} = \set{\tria^*}$ and
	$\modse{\synt{\tria} \uopr \synt{\tria}} = \set{\tria, \tria^*}$, so it
	follows from \eqref{eq:proof:prop:if postulates then ordering} that
	\begin{align*}
		& \tria \notin \min(\set{\tria, \tria^*}, \potrso)
		& \text{and} &
		& \tria \in
			\min(\set{\tria, \tria^*}, \potro)
			\cup
			\min(\set{\tria, \tria^*}, \potrso)
		\enspace.
	\end{align*}
	Hence, $\tria \in \min(\set{\tria, \tria^*}, \potro)$. In other words, if
	$\tria^* \potro \tria$, then it must also be the case that $\tria \potro
	\tria^*$. Consequently, the order assignment generated by $\uopr$ is
	semi-faithful.

	To show that it is also organised, consider well-defined sets of
	\SE-interpretations $\stria$, $\strib$, and \SE-interpretations $\tria$,
	$\trib$ such that
	\begin{align*}
		& \trib \in
			\min \left( \stria, \potro \right)
			\cup \min \left( \stria, \potrso \right)
		& \text{and} &
		& \trib \in
			\min \left( \strib, \potro \right)
			\cup \min \left( \strib, \potrso \right)
		\enspace.
	\end{align*}
	By \eqref{eq:proof:prop:if postulates then ordering} we obtain that $\trib
	\in \modse{\synt{\tria} \uopr \synt{\stria}}$ and $\trib \in
	\modse{\synt{\tria} \uopr \synt{\strib}}$. Applying \puse{7} and \puse{4}
	yields that $\trib \in \modse{\synt{\tria} \uopr \synt{\stria \cup
	\strib}}$. Consequently, by \eqref{eq:proof:prop:if postulates then
	ordering}, either $\trib \in \min \left( \stria \cup \strib, \potro \right)$
	or $\trib \in \min \left( \stria \cup \strib, \potrso \right)$, so the order
	assignment generated by $\uopr$ is organised.
\end{proof}

\begin{lemma}
	\label{lemma:semi-faithful consequence}
	Let $\oas$ be a semi-faithful preorder assignment and $\tria$ an
	\SE-interpretation. Then there is no \SE-interpretation $\trib$ such that
	$\trib \spotrd \tria$.
\end{lemma}
\begin{proof}
	We prove by contradiction. Suppose that $\trib \spotrd \tria$ for some
	\SE-interpretation $\trib$. Clearly, $\trib \neq \tria$ due to irreflexivity
	of $\spotrd$ and $\trib \neq \tria^*$ due to the second condition of
	semi-faithfulness. Hence, $\trib \neq \tria$ and $\trib \neq \tria^*$, so by
	the first condition of semi-faithfulness, either $\tria \spotrd \trib$ or
	$\tria^* \spotrd \trib$. The former is in conflict with the irreflexivity of
	$\spotrd$ and in the latter case it follows by transitivity of $\spotrd$
	that $\tria^* \spotrd \tria$, contrary to the second condition of
	semi-faithfulness.
\end{proof}

\begin{proposition}
	\label{prop:if semi-faithful then faithful}
	Let $\uopr$ be a rule update operator. If $\uopr$ is characterised by a
	semi-faithful and organised preorder assignment, then it is also
	characterised by a faithful and organised partial order assignment.
\end{proposition}
\begin{proof*}
	Let $\uopr$ be characterised by a semi-faithful and organised preorder
	assignment $\oas$. We define the assignment $\oas'$ over $\tris$ as follows:
	\[
		\trib \potrp \tric
		\quad \text{ if and only if } \quad
		\trib = \tria \lor \trib = \tric \lor \trib \spotrd \tric
		\enspace.
	\]

	We need to show that $\potrp$ is a partial order for all $\tria \in \tris$,
	that $\oas'$ is faithful and organised and that for all programs $\prga$,
	$\prgu$,
	\[
		\modse{\prga \uopr \prgu}
		=
		\bigcup_{\tria \in \modse{\prga}}
		\min \left( \modse{\prgu}, \potrp \right)
		\enspace.
	\]

	Note that due to Lemma~\ref{lemma:semi-faithful consequence}, the following
	holds for all \SE-interpretations $\tria$, $\trib$:
	\begin{equation}
		\label{eq:proof:prop:if semi-faithful then faithful}
		\text{If } \trib \potrp \tria \text{, then } \trib = \tria.
	\end{equation}
	Otherwise we would obtain that $\trib \spotrd \tria$ which is in conflict
	with Lemma \ref{lemma:semi-faithful consequence}.

	Turning back to the main proof, reflexivity of $\potrp$ follows directly by
	its definition.

	To show that $\potrp$ is antisymmetric, take some \SE-interpretations
	$\trib_1, \trib_2$ such that $\trib_1 \potrp \trib_2$ and $\trib_2 \potrp
	\trib_1$. If $\trib_1 = \tria$, then $\trib_2 \potrp \tria$ and it follows
	from \eqref{eq:proof:prop:if semi-faithful then faithful} that $\trib_2 =
	\tria = \trib_1$. The case when $\trib_2 = \tria$ is symmetric. If $\trib_1
	\neq \tria$ and $\trib_2 \neq \tria$, then, by the definition of $\potrp$,
	either $\trib_1 = \trib_2$ as desired, or $\trib_1 \spotrd \trib_2$ and
	$\trib_2 \spotrd \trib_1$, which is in conflict with the transitivity and
	irreflexivity of $\spotrd$.

	Turning to transitivity of $\potrp$, suppose that $\trib_1 \potrp \trib_2$
	and $\trib_2 \potrp \trib_3$. We need to show that $\trib_1 \potrp \trib_3$.
	We consider three cases:
	\begin{enumerate}[a)]
		\item If $\trib_1 = \tria$, then $\trib_1 \potrp \trib_3$ by the
			definition of $\potrp$.

		\item If $\trib_2 = \tria$, then $\trib_1 \potrp \tria$, so $\trib_1 =
			\tria$ due to \eqref{eq:proof:prop:if semi-faithful then faithful} and
			the previous case applies.

		\item If $\trib_1 \neq \tria$ and $\trib_2 \neq \tria$, then the desired
			conclusion follows from the transitivity of equality and of $\spotrd$.
	\end{enumerate}

	As for faithfulness of $\oas'$, suppose that $\trib \neq \tria$. We have
	$\tria \potrp \trib$ by definition and $\trib \nleq^\tria_{\oas'} \tria$
	follows from \eqref{eq:proof:prop:if semi-faithful then faithful}.

	To show that $\oas'$ is organised, we prove the following property: For any
	well-defined set of \SE-interpretations $\stri$ and any \SE-interpretation
	$\tria$,
	\begin{equation}
		\label{eq:proof:prop:if semi-faithful then faithful:2}
		\min \left( \stri, \potrp \right)
		\cup \min \left( \stri, \potrsp \right)
		=
		\min \left( \stri, \potrd \right)
		\cup \min \left( \stri, \potrsd \right)
		\enspace.
	\end{equation}
	From \eqref{eq:proof:prop:if semi-faithful then faithful:2} it follows that
	since $\oas$ is organised, $\oas'$ must also be.

	Before we prove \eqref{eq:proof:prop:if semi-faithful then faithful:2}, we
	need to note that $\trib \spotrp \tric$ holds if and only if $\trib \potrp
	\tric$ and $\tric \nleq^\tria_{\oas'} \trib$, so according to the definition
	of $\potrp$,
	\[
		\trib \spotrp \tric
		\quad \text{if and only if} \quad
		(\trib = \tria \lor \trib = \tric \lor \trib \spotrd \tric)
		\land
		(\tric \neq \tria \land \tric \neq \trib
			\land \tric \not\spotrd \trib)
		\enspace.
	\]
	Due to Lemma \ref{lemma:semi-faithful consequence} and the transitivity and
	irreflexivity of $\spotrd$, this can be simplified to
	\begin{equation}
		\label{eq:proof:prop:if semi-faithful then faithful:3}
		\trib \spotrp \tric
		\quad \text{if and only if} \quad
		(\trib = \tria \land \trib \neq \tric) \lor \trib \spotrd \tric
		\enspace.
	\end{equation}
	
	Coming back to the proof of \eqref{eq:proof:prop:if semi-faithful then
	faithful:2}, we need to consider three cases:
	\begin{enumerate}[a)]
		\item If $\tria \notin \stri$ and $\tria^* \notin \stri$, then for all
			$\trib, \tric \in \stri$, $\trib \neq \tria$ and $\trib \neq \tria^*$,
			so by \eqref{eq:proof:prop:if semi-faithful then faithful:3},
			\begin{align*}
				\trib \spotrp \tric
				&\text{ if and only if }
				\trib \spotrd \tric
				&\text{and}&
				& \trib \spotrsp \tric
				&\text{ if and only if }
				\trib \spotrsd \tric
				\enspace,
			\end{align*}
			from which the desired conclusion follows directly.

		\item If $\tria \notin \stri$ and $\tria^* \in \stri$, then for all
			$\trib, \tric \in \stri$, $\trib \neq \tria$, so by
			\eqref{eq:proof:prop:if semi-faithful then faithful:3},
			\[
				\trib \spotrp \tric
				\text{ if and only if }
				\trib \spotrd \tric
				\enspace.
			\]
			Consequently, $\min(\stri, \potrp) = \min(\stri, \potrd)$, and by
			\eqref{eq:proof:prop:if semi-faithful then faithful:3} and
			semi-faithfulness of $\oas$ we obtain $\min(\stri, \potrsp) =
			\set{\tria^*} = \min(\stri, \potrsd)$.

		\item If $\tria \in \stri$, then $\tria^* \in \stri$, and by
			\eqref{eq:proof:prop:if semi-faithful then faithful:3} and
			semi-faithfulness of $\oas$,
			\begin{align*}
				\set{\tria}
				\subseteq
				\min \left( \stri, \potrd \right)
				&\subseteq
				\set{\tria, \tria^*} \enspace,
				& \min \left( \stri, \potrsd \right)
				&=
				\set{\tria^*} \enspace,
				& \\
				\min \left( \stri, \potrp \right)
				&=
				\set{\tria} \enspace,
				& \min \left( \stri, \potrsp \right)
				&=
				\set{\tria^*} \enspace,
			\end{align*}
			from which the desired conclusion follows straightforwardly.
	\end{enumerate}

	Finally, it follows from the assumption that $\oas$ characterises $\uopr$
	and from \eqref{eq:proof:prop:if semi-faithful then faithful:2} that
	\begin{align*}
		\modse{\prga \uopr \prgu}
		&= \bigcup_{\tria \in \modse{\prga}}
			\min \left( \modse{\prgu}, \potrd \right) \\
		&= \bigcup_{\tria \in \modse{\prga}}
			\left(
				\min \left( \modse{\prgu}, \potrd \right)
				\cup \min \left( \modse{\prgu}, \potrsd \right)
			\right) \\
		&= \bigcup_{\tria \in \modse{\prga}}
			\left(
				\min \left( \modse{\prgu}, \potrp \right)
				\cup \min \left( \modse{\prgu}, \potrsp \right)
			\right) \\
		&= \bigcup_{\tria \in \modse{\prga}}
			\min \left( \modse{\prgu}, \potrp \right)
			\enspace. \qedhere
	\end{align*}
\end{proof*}

\begin{proposition}
	\label{prop:if ordering then postulates}
	Let $\uopr$ be a rule update operator. If $\uopr$ is characterised by a
	faithful and organised partial order assignment, then $\uopr$ satisfies
	conditions \puse{1} -- \puse{8}.
\end{proposition}
\begin{proof*}
	Let $\uopr$ be characterised by a faithful and organised partial order
	assignment $\oas$. We consider each condition separately:
	\begin{enumerate}[a))))))]
		\renewcommand{\labelenumi}{\puse{\arabic{enumi}}}

		\item Since $\oas$ characterises $\uopr$, for all programs $\prga$,
			$\prgu$,
			\[
				\modse{\prga \uopr \prgu}
				=
				\bigcup_{\tria \in \modse{\prga}}
				\min \left( \modse{\prgu}, \potrd \right)
				\enspace,
			\]
			so all elements of $\modse{\prga \uopr \prgu}$ belong to
			$\modse{\prgu}$. Equivalently, $\prga \uopr \prgu \entSE \prgu$.

		\item Suppose that $\prga \entSE \prgu$ and take some $\tria \in
			\modse{\prga} \subseteq \modse{\prgu}$. Since the preorder assignment is
			faithful, for all $\trib \in \modse{\prgu}$ with $\trib \neq \tria$ we
			have $\tria \spotrd \trib$. Consequently, $\min(\modse{\prgu}, \potrd) =
			\set{\tria}$ and so
			\[
				\modse{\prga \uopr \prgu}
				=
				\bigcup_{\tria \in \modse{\prga}}
				\min \left( \modse{\prgu}, \potrd \right)
				=
				\bigcup_{\tria \in \modse{\prga}} \set{\tria}
				=
				\modse{\prga}
				\enspace.
			\]

		\item Suppose that both $\modse{\prga} \neq \emptyset$ and $\modse{\prgu}
			\neq \emptyset$. Then there is some $\tria_0 \in \modse{\prga}$ and also
			some $\trib \in \min(\modse{\prgu}, \pod{\tria_0})$, so we obtain
			\[
				\trib
				\in
				\min \left( \modse{\prgu}, \pod{\tria_0} \right)
				\subseteq
				\bigcup_{\tria \in \modse{\prga}}
				\min \left( \modse{\prgu}, \potrd \right)
				=
				\modse{\prga \uopr \prgu}
				\enspace.
			\]
			Hence, $\modse{\prga \uopr \prgu} \neq \emptyset$.

		\item If $\prga \eqSE \prgb$ and $\prgu \eqSE \prgv$, then
			\begin{align*}
				\modse{\prga \uopr \prgu}
				&=
				\bigcup_{\tria \in \modse{\prga}}
				\min \left( \modse{\prgu}, \potrd \right)
				=
				\bigcup_{\tria \in \modse{\prgb}}
				\min \left( \modse{\prgv}, \potrd \right) \\
				&=
				\modse{\prgb \uopr \prgv} \enspace.
			\end{align*}
			Therefore, $\prga \uopr \prgu \eqSE \prgb \uopr \prgv$.

		\item Suppose that $\trib$ is an \SE-model of $(\prga \uopr \prgu) \prgand
			\prgv$. Then $\trib \in \modse{\prgv}$ and there is some \SE-model
			$\tria$ of $\prga$ such that $\trib$ belongs to $\min(\modse{\prgu},
			\potrd)$. Consequently, $\trib$ also belongs to $\min(\modse{\prgu} \cap
			\modse{\prgv}, \potrd)$, so $\trib$ is an \SE-model of $\prga \uopr
			(\prgu \prgand \prgv)$.

		\item Assume that $\prga \uopr \prgu \entSE \prgv$ and $\prga \uopr \prgv
			\entSE \prgu$. We will prove by contradiction that $\prga \uopr \prgu
			\entSE \prga \uopr \prgv$. The other half can be proved similarly.

			So suppose that $\trib$ is an \SE-model of $\prga \uopr \prgu$ but not
			of $\prga \uopr \prgv$. Then there is some \SE-model $\tria$ of $\prga$
			such that
			\begin{equation}
				\label{eq:proof:KM 6}
				\trib \in \min(\modse{\prgu}, \potrd) \enspace.
			\end{equation}
			At the same time, there must be some \SE-model $\tric$ of $\prgv$ such
			that $\tric \spotrd \trib$. Let $\tric_0$ be minimal w.r.t.\ $\potrd$
			among all such $\tric$. Then by transitivity of $\spotrd$ we obtain that
			$\tric_0 \in \min(\modse{\prgv}, \potrd)$ and, consequently, $\tric_0$
			is an \SE-model of $\prga \uopr \prgv$. By the assumption we now obtain
			that $\tric_0$ is an \SE-model of $\prgu$. But since $\tric_0 \spotrd
			\trib$, this is in conflict with \eqref{eq:proof:KM 6}.

		\item Suppose that $\prga$ is strongly equivalent to $\synt{\tria}$ for
			some \SE-interpretation $\tria$ and $\trib$ is an \SE-model of both
			$\prga \uopr \prgu$ and $\prga \uopr \prgv$. We will show that $\trib$
			is an \SE-model of $\prga \uopr (\prgu \prgor \prgv)$. Let $\stria =
			\modse{\prgu}$ and $\strib = \modse{\prgv}$. It follows that
			\[
				\trib \in \min ( \stria, \potrd )
					\cup \min ( \stria, \potrsd )
				\text{ and }
				\trib \in \min ( \strib, \potrd )
					\cup \min ( \strib, \potrsd ) \enspace,
			\]
			so since $\oas$ is organised, $\trib \in \min(\stria \cup \strib,
			\potrd) \cup \min(\stria \cup \strib, \potrsd)$. Consequently, $\trib$
			is an \SE-model of $\prga \uopr (\prgu \prgor \prgv)$.

		\item The following sequence of equations establishes the property:
			\begin{align*}
				\modse{(\prga \prgor \prgb) \uopr \prgu}
				&=
				\bigcup_{\tria \in \modse{\prga \prgor \prgb}}
				\min \left( \modse{\prgu}, \potrd \right) \\
				&=
				\bigcup_{\tria \in \modse{\prga}}
				\min \left( \modse{\prgu}, \potrd \right)
				\cup
				\bigcup_{\tria \in \modse{\prgb}}
				\min \left( \modse{\prgu}, \potrd \right) \\
				&=
				\modse{\prga \uopr \prgu}
				\cup \modse{\prgb \uopr \prgu} \\
				&=
				\modse{(\prga \uopr \prgu) \prgor (\prgb \uopr \prgu)}
				\qedhere
			\end{align*}
	\end{enumerate}
\end{proof*}

\begin{theorem*}{thm:representation}
	Let $\uopr$ be a rule update operator. The following conditions are
	equivalent:
	\begin{enumerate}[a)]
		\setlength{\itemindent}{-1.1em}
		\item The operator $\uopr$ satisfies conditions \puse{1} -- \puse{8}.

		\item The operator $\uopr$ is characterised by a semi-faithful and
			organised preorder assignment.

		\item The operator $\uopr$ is characterised by a faithful and organised
			partial order assignment.
	\end{enumerate}
\end{theorem*}
\begin{proof}
	[\textit{Proof of Theorem~\ref{thm:representation}}]
	\label{proof:representation}
	Follows from Propositions~\ref{prop:if postulates then ordering},
	\ref{prop:if semi-faithful then faithful} and \ref{prop:if ordering then
	postulates}.
\end{proof}

\section{Proofs: Properties of the Assignment $\oasa$}

\begin{proposition}
	\label{prop:oasa is preorder assignment}
	The assignment $\oasa$ is a preorder assignment.
\end{proposition}
\begin{proof}
	Recall that the assignment $\oasa$ is defined for all \SE-interpretations
	$\tria = \twiab$, $\trib = \tpl{\twic_1, \twid_1}$, $\tric = \tpl{\twic_2,
	\twid_2}$ as follows: $\trib \potra \tric$ if and only if
	\begin{enumerate}[1.]
		\item $(\twid_1 \div \twib) \subseteq (\twid_2 \div \twib)$;

		\item If $(\twid_1 \div \twib) = (\twid_2 \div \twib)$, then $(\twic_1
			\div \twia) \setminus \Delta \subseteq (\twic_2 \div \twia) \setminus
			\Delta$ where $\Delta = \twid_1 \div \twib$.
	\end{enumerate}
	In order to show that $\oasa$ is a preorder assignment, we need to prove
	that given an arbitrary \SE-interpretation $\tria = \twiab$, $\potra$ is a
	preorder over $\tris$. This holds if and only if $\potra$ is reflexive and
	transitive. First we show reflexivity. Take some \SE-interpretation $\trib =
	\tpl{\twic, \twid}$. By definition, $\trib \potra \trib$ holds if and only
	if
	\begin{enumerate}[1.]
		\item $(\twid \div \twib) \subseteq (\twid \div \twib)$;

		\item If $(\twid \div \twib) = (\twid \div \twib)$, then $(\twic \div
			\twia) \setminus \Delta \subseteq (\twic \div \twia) \setminus \Delta$
			where $\Delta = \twid \div \twib$.
	\end{enumerate}
	It is not difficult to check that both conditions hold.

	To show transitivity, take some \SE-interpretations $\trib_1 = \tpl{\twic_1,
	\twid_1}, \trib_2 = \tpl{\twic_2, \twid_2}, \trib_3 = \tpl{\twic_3, \twid_3}$
	such that $\trib_1 \potra \trib_2$ and $\trib_2 \potra \trib_3$. We need to
	show that $\trib_1 \potra \trib_3$. According to the definition of $\potra$
	we obtain
	\begin{enumerate}[1.]
		\item $(\twid_1 \div \twib) \subseteq (\twid_2 \div \twib)$;

		\item If $(\twid_1 \div \twib) = (\twid_2 \div \twib)$, then $(\twic_1
			\div \twia) \setminus \Delta \subseteq (\twic_2 \div \twia) \setminus
			\Delta$ where $\Delta = \twid_1 \div \twib$;
	\end{enumerate}
	and also
	\begin{enumerate}[1']
		\item $(\twid_2 \div \twib) \subseteq (\twid_3 \div \twib)$;

		\item If $(\twid_2 \div \twib) = (\twid_3 \div \twib)$, then $(\twic_2
			\div \twia) \setminus \Delta \subseteq (\twic_3 \div \twia) \setminus
			\Delta$ where $\Delta = \twid_2 \div \twib$.
	\end{enumerate}
	We need to show the following two conditions:
	\begin{enumerate}[1$^*$]
		\item $(\twid_1 \div \twib) \subseteq (\twid_3 \div \twib)$;

		\item If $(\twid_1 \div \twib) = (\twid_3 \div \twib)$, then $(\twic_1
			\div \twia) \setminus \Delta \subseteq (\twic_3 \div \twia) \setminus
			\Delta$ where $\Delta = \twid_1 \div \twib$.
	\end{enumerate}
	It can be seen that 1$^*$ follows from 1.\ and 1' by transitivity of the
	subset relation. To show that 2$^*$ holds as well, suppose that $(\twid_1
	\div \twib) = (\twid_3 \div \twib)$. Then by 1.\ and 1' we obtain that
	$(\twid_1 \div \twib) = (\twid_2 \div \twib) = (\twid_3 \div \twib) =
	\Delta$ and so by 2.\ and 2' it holds that
	\[
		(\twic_1 \div \twia) \setminus \Delta
		\subseteq
		(\twic_2 \div \twia) \setminus \Delta
		\subseteq
		(\twic_3 \div \twia) \setminus \Delta
		\enspace.
	\]
	Consequently, 2$^*$ is also satisfied and the proof is finished.
\end{proof}

\begin{lemma}
	\label{lemma:oasa strict}
	Let $\tria = \twiab$, $\trib = \tpl{\twic_1, \twid_1}$, $\tric =
	\tpl{\twic_2, \twid_2}$ be \SE-interpretations. Then $\trib \spotra \tric$
	holds if and only if one of the following conditions is satisfied:
	\begin{enumerate}[a)]
		\item $(\twid_1 \div \twib) \subsetneq (\twid_2 \div \twib)$, or

		\item $(\twid_1 \div \twib) = (\twid_2 \div \twib)$ and $(\twic_1 \div
			\twia) \setminus \Delta \subsetneq (\twic_2 \div \twia) \setminus
			\Delta$ where $\Delta = \twid_1 \div \twib$.
	\end{enumerate}
\end{lemma}
\begin{proof}
	By definition, $\trib \spotra \tric$ holds if and only if $\trib \potra
	\tric$ and it is not the case that $\tric \potra \trib$. This in turn holds
	if and only if the following two conditions hold
	\begin{enumerate}[1.]
		\item $(\twid_1 \div \twib) \subseteq (\twid_2 \div \twib)$;

		\item If $(\twid_1 \div \twib) = (\twid_2 \div \twib)$, then $(\twic_1
			\div \twia) \setminus \Delta \subseteq (\twic_2 \div \twia) \setminus
			\Delta$ where $\Delta = \twid_1 \div \twib$.
	\end{enumerate}
	and one of the following conditions also holds:
	\begin{enumerate}[i)]
		\item $(\twid_2 \div \twib) \nsubseteq (\twid_1 \div \twib)$, or

		\item $(\twid_2 \div \twib) = (\twid_1 \div \twib)$ and $(\twic_2 \div
			\twia) \setminus \Delta \nsubseteq (\twic_1 \div \twia) \setminus
			\Delta$ where $\Delta = \twid_2 \div \twib$.
	\end{enumerate}
	It is not difficult to verify that conditions 1., 2.\ and i) are together
	equivalent to a) and that conditions 1., 2.\ and ii) are together equivalent
	to b). This concludes our proof.
\end{proof}

\begin{proposition}
	\label{prop:oasa is well-defined}
	The assignment $\oasa$ is well-defined.
\end{proposition}
\begin{proof*}
	By definition we need to show that there is a rule update operator
	$\uopr$ such that for all programs $\prga$, $\prgu$,
	\[
		\modse{\prga \uopr \prgu}
		=
		\bigcup_{\tria \in \modse{\prga}}
		\min \left( \modse{\prgu}, \potra \right) \enspace.
	\]
	This holds if and only if for every well-defined set of \SE-interpretations
	$\stri$ and every \SE-interpretation $\tria$, the set of \SE-interpretations
	\begin{equation}
		\label{eq:proof:oasa is well-defined:1}
		\min \left( \stri, \potra \right)
		\cup
		\min \left( \stri, \potrsa \right)
	\end{equation}
	is well-defined. Suppose that $\trib$ belongs to \eqref{eq:proof:oasa is
	well-defined:1}. We need to demonstrate that $\trib^*$ also belongs to
	\eqref{eq:proof:oasa is well-defined:1}. We consider two cases:
	\begin{enumerate}[(a)]
		\item Suppose that $\trib \in \min(\stri, \potra)$. If $\trib^*$ belongs
			to $\min(\stri, \potrsa)$, then we are finished. On the other hand, if
			$\trib^*$ does not belong to $\min(\stri, \potrsa)$, then there must be
			some $\tric \in \stri$ such that $\tric \spotrsa \trib^*$. Let $\trib =
			\tpl{\twic_1, \twid_1}, \tric = \tpl{\twic_2, \twid_2}$ and $\tria =
			\twiab$. By Lemma~\ref{lemma:oasa strict} we know that $\tric \spotrsa
			\trib^*$ holds if and only if one of the following conditions is
			satisfied:
			\begin{enumerate}[a)]
				\item $(\twid_2 \div \twib) \subsetneq (\twid_1 \div \twib)$, or

				\item $(\twid_2 \div \twib) = (\twid_1 \div \twib)$ and $(\twic_2 \div
					\twib) \setminus \Delta \subsetneq (\twid_1 \div \twib) \setminus
					\Delta$ where $\Delta = \twid_2 \div \twib$.
			\end{enumerate}
			If a) is satisfied, then Lemma~\ref{lemma:oasa strict} implies that $\tric
			\spotra \trib$ which is in conflict with the assumption that $\trib \in
			\min(\stri, \potra)$. So b) must hold. But in that case we infer that
			$(\twic_2 \div \twib) \setminus \Delta$ is a proper subset of
			\[
				(\twid_1 \div \twib) \setminus \Delta
				=
				(\twid_1 \div \twib) \setminus (\twid_1 \div \twib)
				=
				\emptyset
				\enspace,
			\]
			which is impossible.

		\item Suppose that $\trib \in \min(\stri, \potrsa)$ and let $\tria =
			\twiab$, $\trib = \tpl{\twic, \twid}$. First we show that $\trib^*
			\potrsa \trib$ holds -- for this, the following conditions need to be
			satisfied:
			\begin{enumerate}[1.]
				\item $(\twid \div \twib) \subseteq (\twid \div \twib)$;

				\item If $(\twid \div \twib) = (\twid \div \twib)$, then $(\twid \div
					\twib) \setminus \Delta \subseteq (\twic \div \twib) \setminus
					\Delta$ where $\Delta = \twid \div \twib$.
			\end{enumerate}
			It is not difficult to verify that both conditions hold.

			Thus, since $\trib^* \potrsa \trib$, there can be no $\tric \in \stri$
			with $\tric \spotrsa \trib^*$ because by transitivity we would obtain
			$\tric \spotrsa \trib$ which would be in conflict with the assumption
			that $\trib \in \min(\stri, \potrsa)$. So $\trib^* \in \min(\stri,
			\potrsa)$ and our proof is finished. \qedhere
	\end{enumerate}
\end{proof*}

\begin{proposition}
	\label{prop:oasa is faithful}
	The assignment $\oasa$ is faithful.
\end{proposition}
\begin{proof*}
	Take some \SE-interpretations $\tria = \twiab$, $\trib = \tpl{\twic, \twid}$
	such that $\trib \neq \tria$. We need to show that $\tria \spotra \trib$. By
	Lemma~\ref{lemma:oasa strict} this holds if and only if one of the following
	conditions is satisfied:
	\begin{enumerate}[a)]
		\item $(\twib \div \twib) \subsetneq (\twid \div \twib)$, or

		\item $(\twib \div \twib) = (\twid \div \twib)$ and $(\twia \div \twia)
			\setminus \Delta \subsetneq (\twic \div \twia) \setminus \Delta$ where
			$\Delta = \twib \div \twib$.
	\end{enumerate}
	We consider two cases:
	\begin{enumerate}[i)]
		\item If $\twid \div \twib = \emptyset$, then $\twid = \twib$ and since
			$\trib \neq \tria$, we conclude that $\twic \neq \twia$. Consequently,
			the second condition is satisfied because $\twia \div \twia = \twib \div
			\twib = \emptyset$ and $\twic \div \twia$ is non-empty.

		\item If $\twid \div \twib \neq \emptyset$, then a) holds since $\twib
			\div \twib = \emptyset$. \qedhere
	\end{enumerate}
\end{proof*}

\begin{proposition}
	\label{prop:oasa is organised}
	The assignment $\oasa$ is organised.
\end{proposition}
\begin{proof*}
	Recall that by definition $\oasa$ is organised if for all
	\SE-interpretations $\tria$, $\trib$ and all well-defined sets of
	\SE-interpretations $\stria, \strib$ the following condition is satisfied:
	\begin{align*}
		& \text{If }
		\trib \in \min ( \stria, \potra )
			\cup \min ( \stria, \potrsa )
		\text{ and }
		\trib \in \min ( \strib, \potra )
			\cup \min ( \strib, \potrsa ), \\
		& \text{then }
		\trib \in \min ( \stria \cup \strib, \potra )
			\cup \min (\stria \cup \strib, \potrsa ).
	\end{align*}
	Suppose that $\trib \notin \min(\stria \cup \strib, \potra) \cup
	\min(\stria \cup \strib, \potrsa)$. We need to show that at least one of
	the following holds:
	\begin{enumerate}[i)]
		\item $\trib \notin \min(\stria, \potra) \cup \min(\stria, \potrsa)$;

		\item $\trib \notin \min(\strib, \potra) \cup \min(\strib, \potrsa)$.
	\end{enumerate}

	If $\trib \notin \stria$, then i) is trivially satisfied. Similarly, if
	$\trib \notin \strib$, then ii) is trivially satisfied. So we can assume
	that $\trib \in \stria \cap \strib$. It follows from the assumption that
	there must be some $\tric_1, \tric_2 \in \stria \cup \strib$ such that
	$\tric_1 \spotra \trib$ and $\tric_2 \spotrsa \trib$. If $\tric_1$ and
	$\tric_2$ both belong to $\stria$, then i) is satisfied; if they both belong
	to $\strib$, then ii) is satisfied. So let's assume, without loss of
	generality, that $\tric_1 \in \stria$ and $\tric_2 \in \strib$.
	Furthermore, let $\tria = \twiab$, $\trib = \tpl{\twic, \twid}$, $\tric_1 =
	\tpl{\twic_1, \twid_1}$ and $\tric_2 = \tpl{\twic_2, \twid_2}$. It follows
	from $\tric_2 \spotrsa \trib$ and Lemma~\ref{lemma:oasa strict} that we need
	to consider two cases:
	\begin{enumerate}[a)]
		\item If $(\twid_2 \div \twib) \subsetneq (\twid \div \twib)$, then by
			Lemma~\ref{lemma:oasa strict} we also have $\tric_2 \spotra \trib$ and,
			consequently, ii) is satisfied.

		\item If $(\twid_2 \div \twib) = (\twid \div \twib)$ and $(\twic_2 \div
			\twib) \setminus \Delta \subsetneq (\twic \div \twib) \setminus \Delta$
			where $\Delta = \twid_2 \div \twib$, then it follows that $(\twic \div
			\twib) \setminus \Delta \neq \emptyset$ and by using $\Delta = \twid_2
			\div \twib = \twid \div \twib$ we obtain
			\begin{equation}
				\label{eq:proof:oasa is organised:1}
				(\twic \div \twib) \setminus (\twid \div \twib)
				\neq
				\emptyset
				\enspace.
			\end{equation}
			Furthermore, from $\tric_1 \spotra \trib$ we know that one of the
			following cases occurs:
			\begin{enumerate}[a')]
				\item $(\twid_1 \div \twib) \subsetneq (\twid \div \twib)$, or

				\item $(\twid_1 \div \twib) = (\twid \div \twib)$ and $(\twic_1 \div
					\twia) \setminus \Delta \subsetneq (\twic \div \twia) \setminus
					\Delta$, where $\Delta = \twid_1 \div \twib$.
			\end{enumerate}
			We will show that $\tric_1^* \spotrsa \trib$. By Lemma~\ref{lemma:oasa
			strict} this holds if and only if one of the following conditions is
			satisfied:
			\begin{enumerate}[a$^*$)]
				\item $(\twid_1 \div \twib) \subsetneq (\twid \div \twib)$, or

				\item $(\twid_1 \div \twib) = (\twid \div \twib)$ and $(\twid_1 \div
					\twib) \setminus \Delta \subsetneq (\twic \div \twib) \setminus
					\Delta$, where $\Delta = \twid_1 \div \twib$.
			\end{enumerate}
			We see that a') implies a$^*$) and b') together with
			\eqref{eq:proof:oasa is organised:1} implies b$^*$). Also, since
			$\stria$ is well-defined, we have $\tric_1^* \in \stria$, so i) is
			satisfied. \qedhere
	\end{enumerate}
\end{proof*}

\begin{proposition*}{prop:oasa is well-defined, faithful and organised}
	The assignment $\oasa$ is a well-defined, faithful and organised preorder
	assignment.
\end{proposition*}
\begin{proof}
	[\textit{Proof of Proposition~\ref{prop:oasa is well-defined, faithful and
	organised}}]
	\label{proof:oasa is well-defined, faithful and organised}
	Follows by Propositions~\ref{prop:oasa is preorder assignment},
	\ref{prop:oasa is well-defined}, \ref{prop:oasa is faithful} and
	\ref{prop:oasa is organised}.
\end{proof}

\section{Proofs: Computational Complexity of Operators Characterised by $\oasa$}

\begin{definition}
	[Truth value assigned by \SE-interpretation]
	Let $\tri$ be an \SE-interpretation and $\atm$ an atom. We define the
	\emph{truth value assigned by $\tri$ to $\atm$} as follows:
	\[
		\tri(\atm) = \begin{cases}
			\tr & \text{if } \atm \in \twia \enspace; \\
			\un & \text{if } \atm \in \twib \setminus \twia \enspace; \\
			\fa & \text{if } \atm \in \atms \setminus \twib \enspace.
		\end{cases}
	\]
\end{definition}

\begin{definition}
	[Set of relevant atoms]
	Let $\frma$ be a propositional formula. We inductively define the \emph{set
	of atoms relevant to $\frma$}, denoted by $\atmsof{\frma}$, as follows:
	\begin{itemize}
		\item If $\frma$ is $\top$ or $\bot$, then $\atmsof{\frma} = \emptyset$;
		\item If $\frma$ is an atom $\atm$, then $\atmsof{\frma} = \set{\atm}$;

		\item If $\frma$ is of the form $\lnot \frmb$, then $\atmsof{\frma} =
			\atmsof{\frmb}$;

		\item If $\frma$ is of the form $\frmb_1 \land \frmb_2$, $\frmb_1 \lor
			\frmb_2$, $\frmb_1 \lthen \frmb_2$ or $\frmb_1 \eq \frmb_2$, then
			$\atmsof{\frma} = \atmsof{\frmb_1} \cup \atmsof{\frmb_2}$.
	\end{itemize}
	For a logic program $\prg$, $\atmsof{\prg} = \atmsof{\tofrm{\prg}}$.
\end{definition}

\begin{lemma}
	\label{lemma:irrelevant atoms inertia}
	Let $\prga$, $\prgu$ be programs and $\uopr$ a rule update operator
	characterised by $\oasa$. If $\tric$ belongs to $\min(\modse{\prgu},
	\potra)$ for some $\tria \in \modse{\prga}$, then $\tria(\atma) =
	\tric(\atma)$ for all $\atma \in \atms \setminus \atmsof{\prgu}$.
\end{lemma}
\begin{proof*}
	We prove by contradiction. Suppose that our assumptions are satisfied and
	$\tria(\atma) \neq \tric(\atma)$ for some $\atma \in \atms \setminus
	\atmsof{\prgu}$. Let the \SE-interpretation $\trib$ be defined as follows:
	\[
		\trib(\atmb) = \begin{cases}
			\tria(\atmb) & \atmb = \atma \enspace; \\
			\tric(\atmb) & \atmb \neq \atma \enspace.
		\end{cases}
	\]
	First note that since $\tric$ is an \SE-model of $\prgu$ and $\trib$ differs
	from $\tric$ only in the truth value assigned to $\atma$, where $\atma
	\notin \atmsof{\prgu}$, it follows that $\trib$ is also an \SE-model of
	$\prgu$.

	Put $\tria = \twiab$, $\trib = \tpl{\twic_1, \twid_1}$ and $\tric =
	\tpl{\twic_2, \twid_2}$. By assumption, $\tria(\atma)
	\neq \tric(\atma)$, so, by the definition of $\trib$, $\trib(\atma) \neq
	\tric(\atm)$. Thus, one of the following cases occurs:
	\begin{enumerate}[a)]
		\item If $\twid_1 \div \twid_2 = \set{\atma}$, then we immediately obtain
			that $(\twid_1 \div \twib) \div (\twid_2 \div \twib) = \set{\atma}$.
			Since $\trib(\atma) = \tria(\atma)$, we conclude that $\atma \notin
			\twid_1 \div \twib$ and it follows that
			\begin{align*}
				& (\twid_1 \div \twib) \setminus (\twid_2 \div \twib) = \emptyset
				&& \text{and}
				&& (\twid_2 \div \twib) \setminus (\twid_1 \div \twib) = \set{\atma}
				\enspace.
			\end{align*}
			Consequently, $\twid_1 \div \twib \subsetneq \twid_2 \div \twib$, so
			$\trib \spotra \tric$, contrary to the assumption that $\tric$ belongs
			to $\min(\modse{\prgu}, \potra)$.

		\item If $\twic_1 \div \twic_2 = \set{\atma}$, then we obtain that
			$(\twic_1 \div \twia) \div (\twic_2 \div \twia) = \set{\atma}$. Since
			$\trib(\atma) = \tria(\atma)$, we conclude that $\atma \notin \twic_1
			\div \twia$ and it follows that
			\begin{align*}
				& (\twic_1 \div \twia) \setminus (\twic_2 \div \twia) = \emptyset
				&& \text{and}
				&& (\twic_2 \div \twia) \setminus (\twic_1 \div \twia) = \set{\atma}
				\enspace.
			\end{align*}
			Furthermore, assuming that the previous case does not occur, it follows
			that $\twid_1 = \twid_2$, so for $\Delta = \twid_1 \div \twib = \twid_2
			\div \twib$ it holds that $\atma \notin \Delta$ because $\tria(\atma) =
			\tric(\atma)$. Consequently, $(\twic_1 \div \twia) \setminus \Delta
			\subsetneq (\twic_2 \div \twia) \setminus \Delta$, so $\trib \spotra
			\tric$, contrary to the assumption that $\tric$ belongs to
			$\min(\modse{\prgu}, \potra)$. \qedhere
	\end{enumerate}
\end{proof*}

\begin{definition}
	[Truth value substitution]
	Let $\tria = \tpl{\twia, \twib}$ be an \SE-interpretation and $\atm$ an
	atom. We define the \SE-interpretations $\trsub{\tria}{\atma}{\tr}$,
	$\trsub{\tria}{\atma}{\un}$ and $\trsub{\tria}{\atma}{\fa}$ as follows:
	\begin{align*}
		\trsub{\tria}{\atma}{\tr}
		&=
		\tpl{\twia \cup \set{\atma}, \twib \cup \set{\atma}}
		\enspace, \\
		\trsub{\tria}{\atma}{\un}
		&=
		\tpl{\twia \setminus \set{\atma}, \twib \cup \set{\atma}}
		\enspace, \\
		\trsub{\tria}{\atma}{\fa}
		&=
		\tpl{\twia \setminus \set{\atma}, \twib \setminus \set{\atma}}
		\enspace.
	\end{align*}
\end{definition}

\begin{lemma}
	\label{lemma:irrelevant atoms order}
	Let $\tria$, $\trib$, $\tric$ be \SE-interpretations, $\atma$ an atom such
	that $\tria(\atma) = \tric(\atma)$ and $\val$ a truth value. Then,
	\begin{align*}
		& \trib \spotra \tric
		&& \text{implies}
		&& \trsub{\trib}{\atma}{\val}
			\spoa{\trsub{\tria}{\atm}{\val}}
			\trsub{\tric}{\atma}{\val}
		\enspace.
	\end{align*}
\end{lemma}
\begin{proof*}
	Put $\tria = \tpl{\twia, \twib}$, $\trib = \tpl{\twic_1, \twid_1}$ and
	$\tric = \tpl{\twic_2, \twid_2}$. The assumption that $\tria(\atma) =
	\tric(\atma)$ implies that
	\begin{align}
		& \atma \notin \twid_2 \div \twib
		&& \text{and}
		&& \atma \notin \twic_2 \div \twia
		\label{eq:lemma:irrelevant atoms order:1}
		\enspace.
	\end{align}
	Furthermore, if $\trib \spotra \tric$, then, by Lemma~\ref{lemma:oasa strict},
	one of the following two cases occurs:
	\begin{enumerate}[a)]
		\item If $(\twid_1 \div \twib) \subsetneq (\twid_2 \div \twib)$, then it
			follows from \eqref{eq:lemma:irrelevant atoms order:1} that $\atma
			\notin \twid_1 \div \twib$ and we obtain the following:
			\begin{align}
				(\twid_1 \cup \set{\atma}) \div (\twib \cup \set{\atma})
				=
				\twid_1 \div \twib
				&\subsetneq
				\twid_2 \div \twib
				=
				(\twid_2 \cup \set{\atma}) \div (\twib \cup \set{\atma})
				\label{eq:lemma:irrelevant atoms order:2}
				\enspace, \\
				(\twid_1 \setminus \set{\atma}) \div (\twib \setminus \set{\atma})
				=
				\twid_1 \div \twib
				&\subsetneq
				\twid_2 \div \twib
				=
				(\twid_2 \setminus \set{\atma}) \div (\twib \setminus \set{\atma})
				\label{eq:lemma:irrelevant atoms order:3}
				\enspace.
			\end{align}
			Finally, we need to consider two cases depending on $\val$:
			\begin{enumerate}[(i)]
				\item If $\val = \tr$ or $\val = \un$, then the second components of
					the \SE-interpretations $\trsub{\tria}{\atma}{\val}$,
					$\trsub{\trib}{\atma}{\val}$ and $\trsub{\tric}{\atma}{\val}$ are
					$\twib \cup \set{\atma}$, $\twid_1 \cup \set{\atma}$ and $\twid_2
					\cup \set{\atma}$, respectively. Hence, the desired conclusion
					follows from \eqref{eq:lemma:irrelevant atoms order:2} by
					Lemma~\ref{lemma:oasa strict}.

				\item If $\val = \fa$, then the second components of the
					\SE-interpretations $\trsub{\tria}{\atma}{\val}$,
					$\trsub{\trib}{\atma}{\val}$ and $\trsub{\tric}{\atma}{\val}$ are
					$\twib \setminus \set{\atma}$, $\twid_1 \setminus \set{\atma}$ and
					$\twid_2 \setminus \set{\atma}$, respectively. Hence, the desired
					conclusion follows from \eqref{eq:lemma:irrelevant atoms order:3} by
					Lemma~\ref{lemma:oasa strict}.
			\end{enumerate}

		\item If $(\twid_1 \div \twib) = (\twid_2 \div \twib)$ and $(\twic_1 \div
			\twia) \setminus \Delta \subsetneq (\twic_2 \div \twia) \setminus
			\Delta$ where $\Delta = \twid_1 \div \twib$, then $\twid_1 = \twid_2$
			and it follows from \eqref{eq:lemma:irrelevant atoms order:1} that
			$\atma \notin \Delta$ as well as $\atma \notin \twic_1 \div \twia$, so
			we obtain the following:
			\begin{align}
				\begin{split}
					(\twid_1 \cup \set{\atma}) \div (\twib \cup \set{\atma})
					&=
					(\twid_2 \cup \set{\atma}) \div (\twib \cup \set{\atma})
					= \Delta
					\enspace,
				\end{split}
				\label{eq:lemma:irrelevant atoms order:4}
				\\
				\begin{split}
					(\twid_1 \setminus \set{\atma}) \div (\twib \setminus \set{\atma})
					&=
					(\twid_2 \setminus \set{\atma}) \div (\twib \setminus \set{\atma})
					= \Delta
					\enspace,
				\end{split}
				\label{eq:lemma:irrelevant atoms order:5}
				\\
				\begin{split}
					[(\twic_1 \cup \set{\atma}) \div (\twia \cup \set{\atma})]
					\setminus \Delta
					&=
					(\twic_1 \div \twia) \setminus \Delta
					\\
					\subsetneq
					(\twic_2 \div \twia) \setminus \Delta
					&=
					[(\twic_2 \cup \set{\atma}) \div (\twia \cup \set{\atma})]
					\setminus \Delta
					\enspace,
				\end{split}
				\label{eq:lemma:irrelevant atoms order:6}
				\\
				\begin{split}
					[(\twic_1 \setminus \set{\atma}) \div (\twia \setminus \set{\atma})]
					\setminus \Delta
					&=
					(\twic_1 \div \twia) \setminus \Delta
					\\
					\subsetneq
					(\twic_2 \div \twia) \setminus \Delta
					&=
					[(\twic_2 \setminus \set{\atma}) \div (\twia \setminus \set{\atma})]
					\setminus \Delta
					\enspace.
				\end{split}
				\label{eq:lemma:irrelevant atoms order:7}
			\end{align}
			Finally, we need to use Lemma~\ref{lemma:oasa strict}, considering three
			cases depending on $\val$:
			\begin{enumerate}[(i)]
				\item If $\val = \tr$, then the desired conclusion follows from
					\eqref{eq:lemma:irrelevant atoms order:4} and
					\eqref{eq:lemma:irrelevant atoms order:6}.

				\item If $\val = \un$, then the desired conclusion follows from
					\eqref{eq:lemma:irrelevant atoms order:4} and
					\eqref{eq:lemma:irrelevant atoms order:7}.

				\item If $\val = \fa$, then the desired conclusion follows from
					\eqref{eq:lemma:irrelevant atoms order:5} and
					\eqref{eq:lemma:irrelevant atoms order:7}. \qedhere
			\end{enumerate}
	\end{enumerate}
\end{proof*}

\begin{lemma}
	\label{lemma:irrelevant atoms}
	Let $\prga$, $\prgu$ be programs, $\atma$ an atom with $\atma \notin
	\atmsof{\prga} \cup \atmsof{\prgu}$, $\uopr$ a rule update operator
	characterised by $\oasa$ and $\tric$, $\tric'$ be \SE-interpretations such
	that $\tric = \trsub{\tric'}{\atma}{\val}$ for some truth value $\val$.
	Then,
	\begin{align*}
		& \tric \in \modse{\prga \uopr \prgu}
		&& \text{if and only if}
		&& \tric' \in \modse{\prga \uopr \prgu}
		\enspace.
	\end{align*}
\end{lemma}
\begin{proof}
	We prove the direct implication, the converse one follows by the symmetry of
	the claim.

	Suppose that $\tric \in \modse{\prga \uopr \prgu}$ but $\tric' \notin
	\modse{\prga \uopr \prgu}$. Then there is some \SE-interpretation $\tria \in
	\modse{\prga}$ such that $\tric$ belongs to $\min(\modse{\prgu}, \potra)$.
	It follows from Lemma~\ref{lemma:irrelevant atoms inertia} that
	\[
		\tria(\atma) = \tric(\atma) = \val
		\enspace.
	\]
	Put $\tric'(\atma) = \val'$ and let $\tria' = \trsub{\tria}{\atma}{\val'}$.
	Since $\tria'$ differs from $\tria$ only in the truth value assigned to
	$\atma$ and $\atma \notin \atmsof{\prga}$, it follows that $\tria' \in
	\modse{\prga}$. Thus, there exists some \SE-interpretation $\trib'$ such
	that $\trib' \spoa{\tria'} \tric'$ and by Lemma~\ref{lemma:irrelevant atoms
	order} we conclude that
	\[
		\trsub{\trib'}{\atma}{\val}
		\spoa{\trsub{\tria'}{\atma}{\val}}
		\trsub{\tric'}{\atma}{\val}
		\enspace.
	\]
	It remains to observe that $\trsub{\tria'}{\atma}{\val} = \tria$ and
	$\trsub{\tric'}{\atma}{\val} = \tric$, so for $\trib =
	\trsub{\trib'}{\atma}{\val}$ we have
	\[
		\trib
		\spotra
		\tric
		\enspace.
	\]
	Since $\trib$ differs from $\trib'$ only in the truth value assigned to
	$\atma$ and $\atma \notin \atmsof{\prgu}$, it follows that $\trib \in
	\modse{\prgu}$ -- a conflict with the assumption that $\tric$ belongs to
	$\min(\modse{\prgu}, \potra)$.
\end{proof}

\begin{corollary}
	\label{cor:irrelevant atoms}
	Let $\prga$, $\prgu$ be programs, $\uopr$ a rule update operator
	characterised by $\oasa$ and $\tric$, $\tric'$ be \SE-interpretations such
	that $\tric(\atma) = \tric'(\atma)$ for all $\atma \in \atmsof{\prga} \cup
	\atmsof{\prgu}$. Then,
	\begin{align*}
		& \tric \in \modse{\prga \uopr \prgu}
		&& \text{if and only if}
		&& \tric' \in \modse{\prga \uopr \prgu}
		\enspace.
	\end{align*}
\end{corollary}
\begin{proof}
	Suppose that
	\[
		\atms \setminus (\atmsof{\prga} \cup \atmsof{\prgu})
		=
		\set{\atma_1, \atma_2, \dotsc, \atma_\lng}
	\]
	and construct a sequence of \SE-interpretations $\tric_0, \tric_1, \dotsc,
	\tric_\lng$ as follows: $\tric_0 = \tric$ and $\tric_{\lia + 1} =
	\trsub{\tric_\lia}{\atma_\lia}{\tric'(\atma_\lia)}$ for all $\lia$ with $0
	\leq \lia < \lng$. Clearly, $\tric_\lng = \tric'$ and
	Lemma~\ref{lemma:irrelevant atoms} can be used $\lng$ times, for each pair
	$(\tric_\lia, \tric_{\lia + 1})$, to infer the desired result.
\end{proof}

\begin{lemma}
	\label{lemma:positive facts remain two-valued}
	Let $\prga$ be a set of facts, $\prgu$ a program such that $\atmsof{\prgu}
	\subseteq \atmsof{\prga}$, $\uopr$ a rule update operator characterised by
	$\oasa$ and $\tric$ an \SE-interpretation from $\modse{\prga \uopr \prgu}$.
	Then for every atom $\atma$ with $(\atma.) \in \prga$ it holds that
	$\tric(\atma) \neq \un$.
\end{lemma}
\begin{proof*}
	Suppose that $\tric$ belongs to $\modse{\prga \uopr \prgu}$, put $\tric =
	\tpl{\twic, \twid}$ and let
	\[
		\trib
		=
		\tpl{\twic \cap \atmsof{\prga}, \twid \cap \atmsof{\prga}}
		\enspace.
	\]
	It follows by Corollary~\ref{cor:irrelevant atoms} that $\trib$ belongs to
	$\modse{\prga \uopr \prgu}$. Thus, there exists some \SE-interpretation
	$\tria \in \modse{\prga}$ such that $\trib$ belongs to $\min(\modse{\prgu},
	\potra)$. Also, using Lemma~\ref{lemma:irrelevant atoms inertia} we conclude
	that $\tria$ assigns truth values as follows:
	\begin{align*}
		\tria(\atmb) = \begin{cases}
			\tr & (\atmb.) \in \prga \enspace; \\
			\fa & (\lpnot \atmb.) \in \prga \enspace; \\
			\fa & \atmb \in \atms \setminus \atmsof{\prga} \enspace.
		\end{cases}
	\end{align*}
	In other words, $\tria$ is of the form $\tpl{\twib, \twib}$ where $\twib =
	\set{\atmb \in \atms | (\atmb.) \in \prga}$. Furthermore, since $\trib$ belongs to $\modse{\prgu}$, $\trib^* =
	\tpl{\twid \cap \atmsof{\prga}, \twid \cap \atmsof{\prga}}$ also belongs
	there.

	We proceed by contradiction: Suppose that $\tric(\atma) = \un$ for some atom
	$\atma$ with $(\atma.) \in \prga$. Then $\atma \in \twid \setminus \twic$,
	$\atma \in \atmsof{\prga}$ and $\atma \in \twib$ and we reach a conflict
	because $\trib^* \spotra \trib$ follows by Lemma~\ref{lemma:oasa strict} from
	the fact that
	\begin{multline*}
		[(\twid \cap \atmsof{\prga}) \div \twib]
		\setminus
		[(\twid \cap \atmsof{\prga}) \div \twib]
		=
		\emptyset
		\\
		\subsetneq
		\set{\atma}
		\subseteq
		[(\twic \cap \atmsof{\prga}) \div \twib]
		\setminus
		[(\twid \cap \atmsof{\prga}) \div \twib]
		\enspace.
		\qedhere
	\end{multline*}
\end{proof*}

\begin{lemma}
	\label{lemma:oasa vs oasw}
	Let $\prga$, $\prgu$ be programs, $\uopr$ a rule update operator
	characterised by $\oasa$, $\uopb$ a belief update operator characterised by
	$\oasw$ and $\twid$ an interpretation. Then,
	\begin{align*}
		& \tpl{\twid, \twid} \in \modse{\prga \uopr \prgu}
		&& \text{if and only if}
		&& \twid \in \mod{\tofrm{\prga} \uopb \tofrm{\prgu}}
		\enspace.
	\end{align*}
\end{lemma}
\begin{proof*}
	Suppose that $\tpl{\twid, \twid} \in \modse{\prga \uopr \prgu}$. Then
	$\tpl{\twid, \twid}$ belongs to $\min(\modse{\prgu}, \potra)$ for some
	$\tria = \tpl{\twia, \twib} \in \modse{\prga}$. Since $\modse{\prga}$ is a
	well-defined set of \SE-interpretations, we conclude that $\tpl{\twib,
	\twib} \in \modse{\prga}$ and, consequently, $\twib \ent \prga$. We will
	prove that $\twid \in \min(\mod{\tofrm{\prgu}}, \pow{\twib})$. Suppose that
	this is not the case, i.e.\ there is some $\twid' \in \mod{\tofrm{\prgu}}$
	such that $\twid' \spow{\twib} \twid$. In other words, $\twid' \div \twib
	\subsetneq \twid \div \twib$. It follows that $\tpl{\twid', \twid'}$ is an
	\SE-model of $\prgu$ and by Lemma~\ref{lemma:oasa strict} we conclude that
	$\tpl{\twid', \twid'} \spotra \tpl{\twid, \twid}$, contrary to the
	assumption that $\tpl{\twid, \twid}$ belongs to $\min(\modse{\prgu},
	\potra)$.

	To prove the converse implication, assume that $\twid \in \mod{\tofrm{\prga}
	\uopb \tofrm{\prgu}}$. Then there is some interpretation $\twib$ with $\twib
	\ent \prga$ such that $\twid \in \min(\mod{\tofrm{\prgu}}, \pow{\twib})$. It
	follows that $\tria = \tpl{\twib, \twib} \in \modse{\prga}$ and $\tric =
	\tpl{\twid, \twid} \in \modse{\prgu}$. Our goal is to prove that $\tric \in
	\min(\modse{\prgu}, \potra)$. Suppose that this is not the case, i.e.\ there
	is some $\tric' = \tpl{\twic', \twid'} \in \modse{\prgu}$ such that $\tric'
	\spotra \tric$. Note that since $\modse{\prgu}$ is a well-defined set of
	\SE-interpretations, it follows that $\tpl{\twid', \twid'} \in
	\modse{\prgu}$ and thus $\twid' \ent \prgu$. By Lemma~\ref{lemma:oasa
	strict}, one of the following conditions is then satisfied:
	\begin{enumerate}[a)]
		\item If $\twid' \div \twib \subsetneq \twid \div \twib$, then we obtain
			$\twid' \spow{\twib} \twid$, contrary to the assumption that $\twid$
			belongs to $\min(\mod{\tofrm{\prgu}}, \pow{\twib})$.

		\item The case when $\twid' \div \twib = \twid \div \twib$ and $(\twic'
			\div \twib) \setminus \Delta \subsetneq (\twid \div \twib) \setminus
			\Delta$, where $\Delta = \twid \div \twib$, is impossible because the
			set $(\twid \div \twib) \setminus \Delta$ is empty. \qedhere
	\end{enumerate}
\end{proof*}

\begin{proposition}
	\label{prop:oasa vs oasw query answering}
	Let $\prga$ be a set of facts, $\prgb$ and $\prgu$ be programs such that
	$\prgb \subseteq \prga$ and $\atmsof{\prgu} \subseteq \atmsof{\prga}$,
	$\uopr$ a rule update operator characterised by $\oasa$ and $\uopb$ a belief
	update operator characterised by $\oasw$. Then,
	\begin{align*}
		\prga \uopr \prgu &\entSE \prgb
		&& \text{if and only if}
		& \tofrm{\prga} \uopb \tofrm{\prgu} &\ent \tofrm{\prgb}
		\enspace.
	\end{align*}
\end{proposition}
\begin{proof}
	First suppose that $\prga \uopr \prgu \entSE \prgb$ and take some $\twid \in
	\mod{\tofrm{\prga} \uopb \tofrm{\prgu}}$. We need to prove that $\twid \ent
	\prgb$. It follows from Lemma~\ref{lemma:oasa vs oasw} that $\tpl{\twid,
	\twid} \in \modse{\prga \uopr \prgu}$ and our assumption implies that
	$\tpl{\twid, \twid} \ent \prgb$. This means that $\twid \ent \prgb$, so we
	reached the desired conclusion.

	For the converse implication, suppose that $\tofrm{\prga} \uopb
	\tofrm{\prgu} \ent \tofrm{\prgb}$ and take some $\tpl{\twic, \twid} \in
	\modse{\prga \uopr \prgu}$. Our goal is to prove that $\tpl{\twic, \twid}
	\ent \prgb$. Since the set of \SE-interpretations $\modse{\prga \uopr
	\prgu}$ is well-defined, we obtain that $\tpl{\twid, \twid} \in \modse{\prga
	\uopr \prgu}$ and by Lemma~\ref{lemma:oasa vs oasw} it follows that $\twid
	\in \mod{\tofrm{\prga} \uopb \tofrm{\prgu}}$. By our assumption we infer
	that $\twid \ent \prgb$. Thus, for every positive fact $(\atma.)$ from
	$\prgb$ it holds that $\atma \in \twid$ and due to Lemma~\ref{lemma:positive
	facts remain two-valued} also $\atma \in \twic$. Therefore, $\tpl{\twic,
	\twid} \ent (\atma.)$. Similarly, for every negative fact $(\lpnot \atma.)$
	from $\prgb$ it holds that $\atma \notin \twid$ and, hence, $\tpl{\twic,
	\twid} \ent (\lpnot \atma.)$. Consequently, $\tpl{\twic, \twid} \ent \prgb$
	as desired.
\end{proof}

\begin{theorem*}
	[Computational complexity of rule updates characterised by $\oasa$]
	{thm:oasa complexity general}
	Let $\uopr$ be a rule update operator characterised by $\oasa$. Deciding
	whether $\prga \uopr \prgu \entSE \prgb$ for programs $\prga$, $\prgu$,
	$\prgb$ is $\Pi^\ccP_2$-complete. Hardness holds even if $\prga$ is a set of
	positive facts, $\prgu$ is a non-disjunctive program and $\prgb$ contains a
	single fact from $\prga$.
\end{theorem*}
\begin{proof}
	[\textit{Proof of Theorem~\ref{thm:oasa complexity general}}]
	\label{proof:oasa complexity general}
	Hardness can be shown by reducing the problem of query answering for
	Winslett's belief update semantics to the problem of query answering for
	$\uopr$. To do this, we rely on some specifics of the proof of
	Theorem~\ref{thm:oasw complexity general} as it is presented in
	\cite{Eiter1992}. More specifically, Lemma~6.2 (c.f.\ page~250 of
	\cite{Eiter1992}) shows $\Pi^\ccP_2$-hardness of Winslett's belief update
	semantics by taking an instance
	\[
		F = \forall x_1, \dotsc, x_m \exists y_1, \dotsc, y_n : \frmv
	\]
	of $\ccstyle{QBF}_{2, \forall}$ and constructing propositional formulae
	$\frma$, $\frmu$ and $\frmb$ such that
	\begin{align}
		\label{eq:thm:oasa complexity general:1}
		& \text{$F$ is valid}
		&& \text{if and only if}
		&& \frma \uopb \frmu \ent \frmb
		\enspace.
	\end{align}
	In the following we reproduce the definition of $\frma$, $\frmu$ and $\frmb$
	in order to pinpoint their syntactic structure. Then we show how they can be
	encoded as logic programs $\prga$, $\prgu$ and $\prgb$ such that
	\begin{align}
		\label{eq:thm:oasa complexity general:2}
		& \frma \uopb \frmu \ent \frmb
		&& \text{if and only if}
		&& \prga \uopr \prgu \entSE \prgb
		\enspace.
	\end{align}
	However, we omit the proof of the equivalence \eqref{eq:thm:oasa complexity
	general:1} and refer the interested reader to \cite{Eiter1992} for further
	details.

	Formulae $\frma$, $\frmu$ and $\frmb$ can be defined as follows:
	\begin{align*}
		\frma
		&=
		x_1 \land \dotsb \land x_m
		\land z_1 \land \dotsb \land z_m
		\land y_1 \land \dotsb \land y_n
		\land r
		\enspace, \\
		\frmu
		&=
		(x_1 \lequiv \lnot z_1) \land \dotsb \land (x_m \lequiv \lnot z_m)
		\land (r \lthen \frmv)
		\land ((y_1 \lor \dotsb \lor y_n) \lthen r)
		\enspace, \\
		\frmb &= r
		\enspace,
	\end{align*}
	where $z_1, \dotsc, z_m$ and $r$ are fresh propositional variables.
	Moreover, we can assume without loss of generality that $\frmv$ is in
	conjunctive normal form, i.e.\
	\[
		\frmv = \bigland_{\lia = 1}^s (
			\atma_{\lia, 1} \lor \dotsb \lor \atm_{\lia, t_\lia}
			\lor
			\lnot \atmb_{\lia, 1} \lor \dotsb \lor \lnot \atmb_{\lia, u_\lia}
		)
	\]
	where $\atma_{\lia, \lib}$ and $\atmb_{\lia, \lic}$ belong to $\set{x_1,
	\dotsc, x_m, y_1, \dotsc, y_n}$ for all $\lia$, $\lib$, $\lic$. We construct
	programs $\prga$, $\prgu$ and $\prgb$ as follows:
	\begin{align*}
		\prga
		=
		&\set{ (x_\lia.) | 1 \leq \lia \leq m }
		\cup \set{ (z_\lia.) | 1 \leq \lia \leq m }
		\cup \set{ (y_\lia.) | 1 \leq \lia \leq n }
		\cup \set{(r.)}
		\enspace, \\
		\begin{split}
			\prgu
			=
			&\set{
				(x_\lia \lpif \lpnot z_\lia.),
				(\lpnot z_\lia \lpif x_\lia.)
				|
				1 \leq \lia \leq m
			} \\
			& \qquad {}\cup \set{
				(
					\bot
					\lpif
					\lpnot \atma_{\lia, 1}, \dotsc, \lpnot \atm_{\lia, t_\lia},
					\atmb_{\lia, 1}, \dotsc, \atmb_{\lia, u_\lia},
					r.
				)
				|
				1 \leq \lia \leq s
			} \\
			& \qquad {}\cup \set{
				(r \lpif y_\lia.)
				|
				1 \leq \lia \leq n
			} \enspace,
		\end{split}
		\\
		\prgb
		=
		&\set{(r.)}
		\enspace.
	\end{align*}
	It is not difficult to verify that $\tofrm{\prga} \lequiv \frma$,
	$\tofrm{\prgu} \lequiv \frmu$ and $\tofrm{\prgb} \lequiv \frmb$, so it
	follows from postulate \bu{4} and Proposition~\ref{prop:oasa vs oasw query
	answering} that \eqref{eq:thm:oasa complexity general:2} is satisfied.
	Together with \eqref{eq:thm:oasa complexity general:1} this implies that
	query answering for rule update operators characterised by $\oasa$ is
	$\Pi^\ccP_2$-hard.

	To verify membership to $\Pi^\ccP_2$, consider the following
	non-deterministic polynomial algorithm with an $\ccNP$ oracle, analogous to
	the one for Winslett's belief update semantics (c.f.\ proof of Theorem~6.4
	on page~252 in \cite{Eiter1992}): To prove that $\prga \uopr \prgu \nentSE
	\prgb$, consider only atoms from $\atmsof{\prga} \cup \atmsof{\prgu} \cup
	\atmsof{\prgb}$ (this can be done due to Corollary~\ref{cor:irrelevant
	atoms}), guess some \SE-interpretations $\tria$ and $\trib$, check in
	polynomial time that $\tria \in \modse{\prga}$, $\trib \in \modse{\prgu}$
	and $\trib \notin \modse{\prgb}$ and invoke the $\ccNP$ oracle to check that
	there is no $\tric \in \modse{\prgu}$ such that $\tric \spotra \trib$.
\end{proof}

\begin{lemma}
	\label{lemma:se models of definite program}
	Let $\prgu$ be a definite program. Then for all interpretations $\twia$,
	$\twib$ it holds that,
	\begin{align*}
		& \tpl{\twia, \twib} \in \modse{\prgu}
		&& \text{if and only if}
		&& \twia \subseteq \twib
			\land \twia \ent \tofrm{\prgu}
			\land \twib \ent \tofrm{\prgu}
		\enspace.
	\end{align*}
\end{lemma}
\begin{proof}
	Follows from the fact that since $\prgu$ is definite, $\prgu^\twic = \prgu$
	for any interpretation $\twic$.
\end{proof}

\begin{theorem*}
	[Computational complexity of definite rule updates characterised by $\oasa$]
	{thm:oasa complexity definite}
	Let $\uopr$ be a rule update operator characterised by $\oasa$. Deciding
	whether $\prga \uopr \prgu \entSE \prgb$ for definite programs $\prga$,
	$\prgu$, $\prgb$ is $\cccoNP$-complete. Hardness holds even if $\prga$ is a
	set of facts and $\prgb$ contains a single fact from $\prga$.
\end{theorem*}
\begin{proof}
	[\textit{Proof of Theorem~\ref{thm:oasa complexity definite}}]
	\label{proof:oasa complexity definite}
	Hardness follows by reducing the $\cccoNP$-complete problem of query
	answering for Horn formulae under Winslett's belief update semantics. More
	specifically, Theorem~\ref{thm:oasw complexity Horn} shows that deciding
	whether $\frma \uopb \frmu \ent \frmb$, where $\uopb$ is a belief update
	operator characterised by $\oasw$, is $\cccoNP$-hard even when $\frma$ is a
	conjunction of objective literals, $\frmu$ is a Horn formula and $\frmb$ is
	one of the literals in $\frma$. It is straightforward to construct a set of
	facts $\prga$, a definite program $\prgu$ and a program $\prgb$ containing a
	single fact from $\prga$ such that $\tofrm{\prga} \lequiv \frma$,
	$\tofrm{\prgu} \lequiv \frmu$ and $\tofrm{\prgb} \lequiv \frmb$. Finally, it
	follows from postulate \bu{4} and Proposition~\ref{prop:oasa vs oasw query
	answering} that
	\begin{align*}
		& \prga \uopr \prgu \entSE \prgb
		&& \text{if and only if}
		&& \frma \uopr \frmu \ent \frmb
		\enspace,
	\end{align*}
	which concludes the proof of $\cccoNP$-hardness of query answering for
	$\uopr$.

	To verify membership to $\cccoNP$, consider the following non-deterministic
	polynomial algorithm, analogous to the one for Winslett's belief update
	semantics for Horn formulae (c.f.\ proof of Theorem~7.2 on page~259 in
	\cite{Eiter1992}): To prove that $\prga \uopr \prgu \nentSE \prgb$, consider
	only atoms from $\atms' = \atmsof{\prga} \cup \atmsof{\prgu} \cup \atmsof{\prgb}$
	(this can be done due to Corollary~\ref{cor:irrelevant atoms}), guess some
	\SE-interpretations $\tria = \tpl{\twia, \twib}$ and $\trib = \tpl{\twic,
	\twid}$ and check in polynomial time that $\tria \in \modse{\prga}$, $\trib \in
	\modse{\prgu}$ and $\trib \notin \modse{\prgb}$. It remains to check that
	there is no \SE-interpretation $\tric \in \modse{\prgu}$ such that $\tric
	\spotra \trib$. This can be performed in polynomial time by using
	Lemma~\ref{lemma:se models of definite program} as follows:
	Put $\Delta = \twid \div \twib$ and $\Delta' = (\twic \div \twia)
	\setminus \Delta$ and let for every atom $\atm$,
	\begin{align*}
		t(\atm) &= \begin{cases}
			\atm & \twib \ent \atm \enspace; \\
			\lnot \atm & \twib \nent \atm \enspace;
		\end{cases}
		&
		s(\atm) &= \begin{cases}
			\atm & \twia \ent \atm \enspace; \\
			\lnot \atm & \twia \nent \atm \enspace.
		\end{cases}
	\end{align*}
	It follows from Lemma~\ref{lemma:se models of definite program} and from the
	definition of $\potra$ that it suffices to verify that for every $\atma \in
	\Delta$ and every $\atmb \in \Delta'$, both of the Horn formulae
	\begin{align*}
		& \tofrm{\prgu}
			\land t(\atma)
			\land \bigland_{\atmc \in \atms' \setminus \Delta} t(\atmc)
		&& \text{ and}
		&& \tofrm{\prgu}
			\land s(\atmb)
			\land \bigland_{\atmc \in \atms' \setminus \Delta'} s(\atmc)
	\end{align*}
	are not satisfiable.
\end{proof}

\bibliographystyle{acmtrans}
\bibliography{bibliography}

\end{document}